\documentclass[11pt]{article}
\usepackage[utf8]{inputenc}
\usepackage[margin=1in]{geometry}
\usepackage[T1]{fontenc}
\usepackage{lmodern}

\usepackage{microtype}
\usepackage{graphicx}
\usepackage{subfigure}
\usepackage{booktabs}
\usepackage{natbib}
\setcitestyle{open={(},close={)}}

\usepackage{amsmath}
\usepackage{mathtools}
\usepackage{amsthm}
\usepackage{algorithmic}
\usepackage{graphicx}
\usepackage{textcomp}
\usepackage{xcolor}
\usepackage{amsfonts}
\usepackage{epsfig}
\usepackage{algorithm}
\usepackage{wrapfig}
\usepackage{balance}
\usepackage{url}
\usepackage{mathtools}
\usepackage{natbib}
\usepackage{multirow}

\usepackage[textsize=tiny]{todonotes}

\usepackage[colorlinks=true,citecolor=blue]{hyperref}

\usepackage{amsmath,amsthm,amsfonts,amssymb,mathdots,array,mathrsfs,bm,bbm,stmaryrd,graphicx,subfigure,xcolor}

\usepackage{breakcites}

\usepackage[T1]{fontenc}
\usepackage{enumerate}
\usepackage{inputenc}

\usepackage{graphicx} 
\usepackage{subfigure}

\usepackage{booktabs,balance}
\usepackage{rotating}
\usepackage{boldline}
\usepackage{makecell}
\usepackage{multirow}
\usepackage{balance}

\usepackage{tikz}

\newtheorem{theorem}{Theorem}[section]

\newtheorem{lemma}[theorem]{Lemma}

\newtheorem{definition}{Definition}[section]

\newcommand{\bsmat}{\begin{bmatrix} }
\newcommand{\esmat}{\end{bmatrix} }

\usepackage{environ}
\NewEnviron{smallequation}{%
    \begin{equation}
    \scalebox{0.97}{$\BODY$}
    \end{equation}
    }
    \NewEnviron{smallalign}{%
    \begin{equation}
    \scalebox{0.97}{$\BODY$}
    \end{equation}
    }

\begin{document}

\title{\bf\Huge Variational Flow Graphical Model}

\author{\vspace{0.5in}\\\textbf{Shaogang Ren, Belhal Karimi, Dingcheng Li, Ping Li} \\\\
Cognitive Computing Lab\\
Baidu Research\\
10900 NE 8th St. Bellevue, WA 98004, USA\\\\
  \texttt{\{renshaogang, belhal.karimi, dingchengl,  pingli98\}@gmail.com}
}

\date{\vspace{0.5in}}
\maketitle

\begin{abstract}\vspace{0.3in}
\noindent\footnote{This work was initially submitted in 2020.}This paper introduces a novel approach to embed flow-based models with hierarchical structures. The proposed framework is named Variational Flow Graphical (VFG) Model.   VFGs  learn the representation of high dimensional data via a message-passing scheme by integrating flow-based functions through variational inference.   By leveraging the expressive power of neural networks,   VFGs produce a representation of the data using a lower dimension, thus overcoming the drawbacks of many flow-based models, usually requiring a high dimensional latent space involving many trivial variables.  Aggregation nodes are introduced in the VFG models to integrate forward-backward hierarchical information via a message passing scheme. Maximizing  the evidence lower bound (ELBO) of  data likelihood  aligns the forward and backward messages in each aggregation node  achieving a consistency node state. Algorithms have been developed to learn model parameters through gradient updating  regarding the ELBO objective.

\vspace{0.1in}

\noindent The consistency of  aggregation nodes  enable  VFGs   to be applicable in  tractable inference on graphical structures. Besides representation learning and numerical inference,  VFGs provide a new approach for distribution modeling  on datasets with graphical latent structures. Additionally,  theoretical study  shows that VFGs are universal approximators by leveraging  the implicitly invertible flow-based  structures. With flexible graphical structures and superior excessive power, VFGs could potentially be used  to improve  probabilistic inference.

\vspace{0.1in}
\noindent In the experiments,  VFGs   achieves improved evidence lower bound (ELBO) and likelihood values on multiple datasets.  We also highlight the benefits of our VFG model on  missing entry imputation  for  datasets with graph structures.  Multiple experiments on synthetic and real-world datasets confirm the benefits of the proposed method and potentially broad applications.

\end{abstract}

\newpage

\section{Introduction} \label{sec:intro}

Learning tractable  distribution or density functions from datasets has broad applications.
Probabilistic graphical models (PGMs) provide a unifying framework for capturing complex dependencies among random variables~\citep{bishop2006pattern,wainwright2008graphical,koller2009probabilistic}.
There are two general approaches for probabilistic inference with PGMs and other models: exact inference and approximate inference. In most cases, exact inference is either computationally involved or simply intractable.
Variational inference (VI), stemmed from statistical physics, is computationally efficient and is applied to tackle  large-scale inference problems~\citep{anderson1987mean,hinton1993keeping,jordan1999introduction,ghahramani1999variational,hoffman2013stochastic,blei2017variational,fang2021variational}.
In variational inference, mean-field approximation~\citep{anderson1987mean,hinton1993keeping,xing2012generalized} and variational message passing~\citep{bishop2003vibes,winn2005variational} are two common approaches.
These  methods are limited by the choice of distributions that are inherently unable to recover the true posterior, often leading to a loose approximation.

\vspace{0.1in}

To tackle the probabilistic inference problem, alternative models have been developed under the name of \emph{tractable probabilistic models~(TPMs)}. They include probabilistic decision graphs~\citep{jaeger2006learning}, arithmetic circuits~\citep{darwiche2003differential}, and-or search spaces~\citep{marinescu2005and},
multi-valued decision diagrams~\citep{dechter2007and},
 sum-product nets~\citep{sanchez2021sum}, probabilistic sentential decision diagrams~\citep{kisa2014probabilistic}, and probabilistic circuits~(PCs)~\citep{choi2020probabilistic}. PCs leverage the recursive mixture models and distributional factorization to establish tractable probabilistic inference. PCs also aim to attain a TPM with improved expressive~power. The recent GFlowNets~\citep{bengio2021gflownet} also target tractable probabilistic inference on different structures.

\vspace{0.1in}

Apart from probabilistic inference, generative models have been developed to model high dimensional datasets and to learn   meaningful  hidden data representations by leveraging the approximation power of neural networks. These models also provide a possible approach to generate new samples from  underlining distributions. Variational Auto-Encoders (VAEs)~\citep{kingma2013auto} and Generative Adversarial Networks (GAN)~\citep{Goodfellow14,arjovsky2017towards,karras2019style,zhu2017,yin2020meta,ren2020estimate} are widely applied to different categories of datasets. Flow-based models~\citep{Dinh2016DensityEU, dinh2014nice,rezende2015variational,berg2018sylvester,ren2021causal} leverage invertible neural networks and can estimate the density values of data samples as well. Energy-based models (EBMs)~\citep{ZhuWM98,lecun2006tutorial,hinton2012practical,XieLZW16,nijkamp2019learning,zhao2020learning,zheng2021patchwise} define an unnormalized probability density function of data, which is the exponential of the negative energy function. Unlike TPMs, it is usually difficult to directly use generative models to perform probabilistic inference~on~datasets.

\vspace{0.1in}

In this paper, we  introduce \textsc{Variational Flow Graphical (VFG)} models. By leveraging the expressive power of neural networks, VFGs can learn  latent representations from data.
VFGs also follow the stream of tractable neural networks that allow to perform inference on graphical structures. Sum-product networks~\citep{sanchez2021sum} and probabilistic circuits~\citep{choi2020probabilistic} are falling into this type of models as well. Sum-product networks and probabilistic circuits depend on mixture models and probabilistic factorization in graphical structure for inference. Whereas, VFGs rely on the consistency of  aggregation nodes in graphical structures to achieve tractable inference.  Our contributions are summarized as follows.

\vspace{0.2in}
\noindent \textbf{Summary of contributions.}
Dealing with high dimensional data using graph structures exacerbates the systemic inability for effective distribution modeling and efficient inference. To overcome these limitations, we propose the VFG model to achieve the following goals:
\begin{itemize}
\item \textbf{Hierarchical and flow-based:} VFG is a  novel graphical architecture uniting the hierarchical latent structures and flow-based models.  Our model outputs a tractable posterior distribution used as an approximation of the true posterior of the hidden node states in the considered graph structure.

\item \textbf{Distribution modeling:}
Our theoretical analysis shows that VFGs are universal approximators.
In the experiments, VFGs can  achieve improved evidence lower bound (ELBO) and likelihood values by leveraging  the implicitly invertible flow-based model structure.

\item \textbf{Numerical inference:}  Aggregation nodes are introduced in the model to integrate hierarchical information through a variational forward-backward message passing scheme.  We highlight the benefits of our VFG model on  applications: the missing entry imputation problem and the numerical inference on graphical data.
\end{itemize}

\noindent Moreover, experiments  show that our model achieves to disentangle the factors of variation underlying high dimensional input data.

\vspace{0.2in}

\noindent\textbf{Roadmap:} Section~\ref{sec:prelim} presents important concepts used in the paper.
Section~\ref{sec:tech} introduces the Variational Flow Graphical  (VFG) model. The approximation property of VFGs is discussed in Section~\ref{sec:approx}.
Section~\ref{sec:algrithm} provides the algorithms used to train VFG models.
Section~\ref{sec:infer} discusses how to perform inference with a  VFG model. Section~\ref{sec:numerical} showcases the advantages of VFG on various tasks. Section~\ref{sec:discuss} and Section~\ref{sec:conclusion} provide a discussion and conclusion~of~the~paper.

\vspace{0.1in}
\section{Preliminaries}\label{sec:prelim}
\vspace{0.1in}

We introduce the general principles and notations of  variational inference  and  flow-based models in this section.

\vspace{0.2in}

\noindent\textbf{Notation:} We use $[L]$ to denote the set $ \{1, \cdots, L\}$, for all $L >1$.  $\textbf{\text{KL}}(p || q ) := \int_{\mathcal{Z}} p(z) \log(p(z)/q(z)) \mathrm{d}z$  is the Kullback-Leibler divergence from $q$ to $p$, two probability density functions defined on the set $\mathcal{Z} \subset \mathbb{R}^m$ for any dimension $m >0$.

\vspace{0.2in}

\noindent\textbf{Variational Inference:}
Following the setting discussed above, the functional mapping $\mathbf{f}: \mathcal{Z} \xrightarrow{} \mathcal{X} $ can be viewed as a decoding process and the mapping $\mathbf{f}^{-1}$: $ \mathcal{X} \xrightarrow{}  \mathcal{Z}$ as an encoding one between  random variables $\mathbf{z} \in \mathcal{Z}$  and $\mathbf{x} \in \mathcal{X}$ with densities $\mathbf{z} \sim p(\mathbf{z}), \mathbf{x} \sim p_{\theta}(\mathbf{x}|\mathbf{z}).$
To learn the parameters $\theta$, VI employs a parameterized family of so-called variational distributions $q_{\phi}(\mathbf{z}|\mathbf{x})$ to approximate the true posterior $p(\mathbf{z}|\mathbf{x}) \varpropto  p(\mathbf{z})  p_{\theta}(\mathbf{x}|\mathbf{z})$.
The optimization problem of VI can be shown to be equivalent to maximizing the following evidence lower bound (ELBO) objective, noted $\mathcal{L}(\mathbf{x}; \theta, \phi)$:
\begin{align}\label{eq:vi_elbo}
 & \log p(\mathbf{x})
    \geqslant  \mathcal{L}(\mathbf{x}; \theta, \phi) =\mathbb{E}_{q_{\phi}(\mathbf{z}|\mathbf{x})} \big[\log p_{\theta}(\mathbf{x}|\mathbf{z}) \big] - \textbf{\text{KL}}(q_{\phi}(\mathbf{z}|\mathbf{x})||p(\mathbf{z}))  \, .
\end{align}

\newpage

In Variational Auto-Encoders~(VAEs,~\citep{kingma2013auto,rezende2014stochastic}), the calculation of the reconstruction term requires sampling from the posterior distribution along with using the reparameterization trick, i.e.,
\begin{align} \label{eq:vae_recon}
\mathbb{E}_{q_{\phi}(\mathbf{z}|\mathbf{x})} \big[\log p_{\theta}(\mathbf{x}|\mathbf{z}) \big] \simeq \frac{1}{U}\sum_{u=1}^U \log p(\mathbf{x}| \mathbf{z}_{u}). \end{align}
Here $U$ is the number of latent variable samples drawn from the posterior $q_{\phi}(\mathbf{z}|\mathbf{x})$ regarding data $\mathbf{x}$.

\vspace{0.1in}
\noindent\textbf{Flow-based Models:}
Flow-based models~\citep{Dinh2016DensityEU, dinh2014nice,rezende2015variational,berg2018sylvester}
correspond to a probability distribution transformation using  a sequence of invertible and differentiable mappings, noted $\mathbf{f}: \mathcal{Z} \xrightarrow[]{} \mathcal{X}$. By defining the aforementioned invertible maps $\{\mathbf{f}_{\ell} \}_{\ell =1}^L$, and by the chain rule and inverse function theorem, the variable $\mathbf{x}=\mathbf{f}(\mathbf{z})$ has a tractable probability density function~(pdf) given as:
\begin{align}\label{eq:flow}
\log p_{\theta}(\mathbf{x}) =  \log p(\mathbf{z}) + \sum_{i=1}^L\log \bigg| \text{det} ( \frac{\partial \mathbf{h}^i } {\partial \mathbf{h}^{i-1}}) \bigg| \, ,
\end{align}
where we have $\mathbf{h}^0 = \mathbf{x}$ and $\mathbf{h}^L = \mathbf{z}$ for conciseness.
The scalar value $\log |\text{det}( \partial \mathbf{h}^i/\partial \mathbf{h}^{i-1})|$ is the logarithm of the absolute value of the determinant of the Jacobian matrix $\partial \mathbf{h}^i/\partial \mathbf{h}^{i-1}$, also called the log-determinant.
Eq.~(\ref{eq:flow}) yields a simple mechanism to build families of distributions that, from an initial density and a succession of invertible transformations, returns tractable density functions that one can sample from. \citet{rezende2015variational} propose an approach to construct flexible posteriors by transforming  a simple base posterior with a sequence of flows. Firstly a stochastic latent variable is draw from base posterior $\mathcal{N}(\mathbf{z}_0|\mathbf{\mu}(\mathbf{x}), \mathbf{\sigma}(\mathbf{x}) )$. With $K$ flows, latent variable $\mathbf{z}_0$ is transformed to $\mathbf{z}_k$.The reformed  EBLO is given by
\begin{align*}
\mathcal{L}(\mathbf{x}; \theta, \phi) & =  \mathbb{E}_{q_{\phi}} \big[\log p_{\theta}(\mathbf{x},\mathbf{z})-\log q_{\phi}(\mathbf{z}|\mathbf{x}) \big] \\
&= \mathbb{E}_{q_{0}} \big[\log p_{\theta}(\mathbf{x},\mathbf{z}) - \log q_{0}(\mathbf{z}_0|\mathbf{x}) \big] + \mathbb{E}_{q_{0}} \big[\sum_{k=1}^K \log\big|\det(\frac{\partial \mathbf{f}_k( \mathbf{z}_k; \psi_k)}{\partial \mathbf{z}_k}) \big| \big].
\end{align*}
Here $\mathbf{f}_k$ is the $k$-th flow with parameter $\psi_k$, i.e., $\mathbf{z}_K = \mathbf{f}_K \circ \cdots  \mathbf{f}_2 \circ  \mathbf{f}_1(\mathbf{z}_0)$. The flows are considered as functions of data sample $\mathbf{x}$, and they determine the final distribution~in~amortized~inference.
Several recent models have been proposed by leveraging the invertible  flow-based models. Graphical normalizing flow~\citep{wehenkel2021graphical}  learns a DAG structure from the input data under  sparse penalty and maximum likelihood estimation.
The bivariate causal discovery method proposed in~\citet{khemakhem2021causal} relies on  autoregressive structure of flow-based models and the asymmetry of  log-likelihood ratio  for cause-effect pairs. In this paper, we propose a framework that generalizes flow-based models~\citep{Dinh2016DensityEU, dinh2014nice,rezende2015variational,berg2018sylvester} to graphical variable inference.

\newpage

\section{Variational Flow Graphical Model}\label{sec:tech}

Assume $\mathbf{k}$ sections in the data samples, i.e., $\mathbf{x} = [\mathbf{x}^{(1)}, ..., \mathbf{x}^{(k)}]$, and  a relationship among these sections and the corresponding  latent variable.
Then, it is possible to define a graphical model using normalizing flows, as introduced Section~\ref{sec:prelim}, leading to exact latent variable inference and log-likelihood evaluation of data samples.

A VFG model  $\mathbb{G}=\{\mathcal{V}, \mathbf{f}\} $ consists of  a  node set  ($\mathcal{V}$) and an edge set ($\mathbf{f}$). An edge can be either a flow function or an identity function. There are two types of nodes in a VFG: \emph{aggregation} nodes and \emph{non-aggregation} nodes.
A non-aggregation node connects with another node with a  flow function or an identity function. An aggregation node has multiple children, and it  connects  each of  them with an identity function. Figure~\ref{fig:tree}-Left gives an illustration of an aggregation node and  Figure~\ref{fig:tree}-Right shows a tree VFG model.
Unlike classical graphical models, a node in a VFG model may represent   a single variable or multiple variables. Moreover, each latent variable belongs to only one node in a VFG. In the following sections, identity function is considered as a special case of flow functions.

\begin{figure}[h]

\vspace{-0.1in}

\begin{center}
 \includegraphics[width=5in]{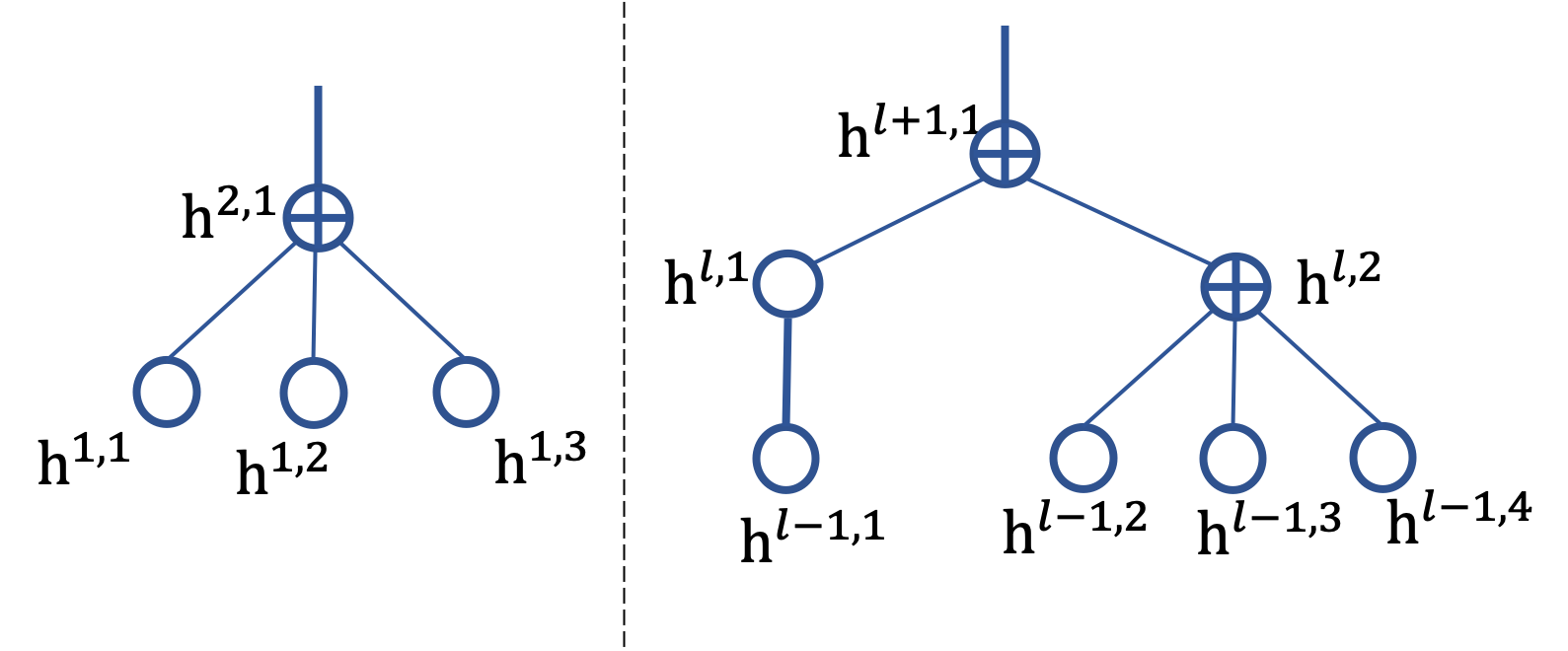}
\end{center}

\vspace{-0.35in}

\caption{(Left)  Node $\mathbf{h}^{2, 1}$ connects its children with invertible functions. Messages from the children are aggregated at the parent node, $\mathbf{h}^{2,1}$. (Right) An illustration of the latent structure from layer $l-1$ to $l+1$.  Thin lines are identity functions, and thick lines are flow functions.   $\oplus$ is an aggregation node, and circles stand for non-aggregation~nodes.}
\label{fig:tree}\vspace{-0.1in}
\end{figure}

\subsection{Evidence Lower Bound of VFGs}

We apply variational inference to learn model parameters  $\theta$ from data samples. Different from VAEs, the recognition model~(encoder) and the generative model~(decoder) in a VFG share the same  neural net structure and parameters. Moreover, the latent variables in a VFG lie in a hierarchy structure and  are generated with deterministic flow functions.

\begin{figure}[t]

\begin{center}
 \includegraphics[width=3.2in]{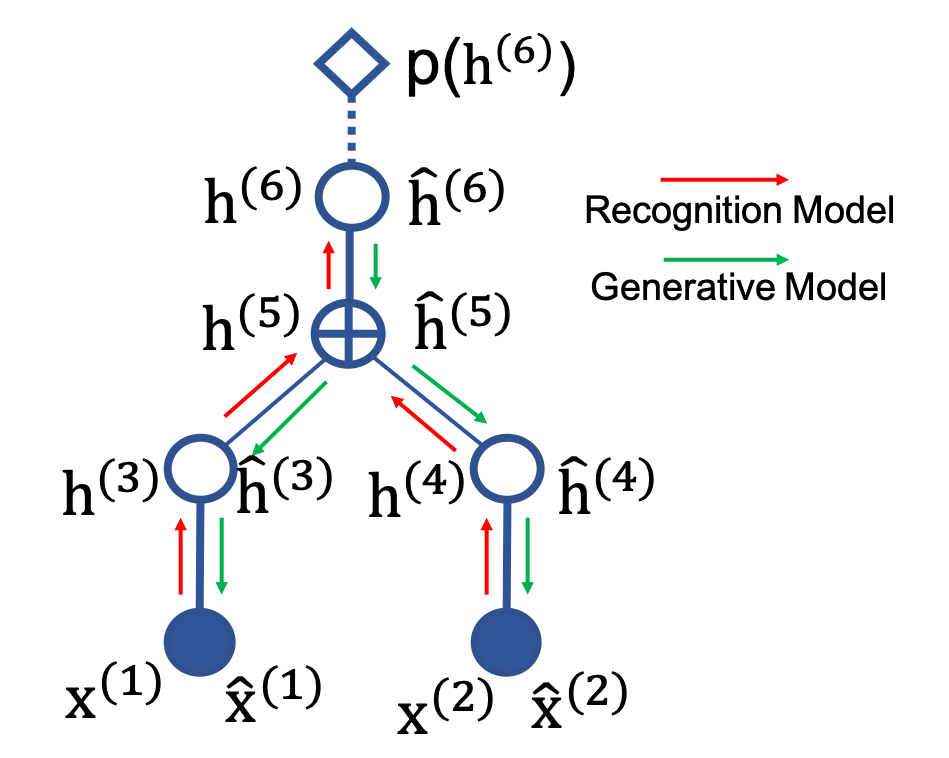}
\end{center}

\vspace{-0.3in}

\caption{Forward message from  data to approximate posterior distributions; generative model is realized by backward message from the root and generates the samples or reconstructions~at~each~layer.}
\label{fig:tree_message}
\end{figure}

We start with a tree VFG (Figure~\ref{fig:tree_message}) to introduce the ELBO of the model. The hierarchical tree structure comprises $L$ layers, $\mathbf{h}^l$ denotes the latent state in layer $l$ of the tree. We use  $\mathbf{h}^{(j)}$ to represent node $j$'s latent state without specification of the layer number, and $j$ is the node index in a tree or graph.
The joint distribution for the hierarchical model is then
\begin{align}\notag
p_{\theta}(\mathbf{x}, \mathbf{h}) = p( \mathbf{h}^{L}) p(\mathbf{h}^{L-1} | \mathbf{h}^{L}) \cdot \cdot  \cdot p(\mathbf{h}^{1} | \mathbf{h}^{2})  p(\mathbf{x} | \mathbf{h}^{1}) \, .
\end{align}
where $\mathbf{h}=\{\mathbf{h}^1, \cdots, \mathbf{h}^L \}$ denotes the set of latent states of the model. The hierarchical generative model is given by factorization $p(\mathbf{x}|\mathbf{h}^L) = p(\mathbf{x} | \mathbf{h}^{1}) \mathbf{\Pi}_{l=1}^{L-1}p(\mathbf{h}^{l} | \mathbf{h}^{l+1})  $, and the prior distribution is $p(\mathbf{h}^L)$. Note that only the \emph{root nodes} have \emph{prior distributions}.
The  probabilistic  density function $p(\mathbf{h}^{l-1} | \mathbf{h}^{l})$ in the generative  model is parameterized with one or multiple invertible  flow functions.  By leveraging the invertible  flow functions, we use variational inference  to approximate the posterior distribution of latent states.
The hierarchical posterior~(recognition model) is factorized as
\begin{align}\label{eq:posterior}
q_{\theta}(\mathbf{h}| \mathbf{x}) =  q(\mathbf{h}^1 | \mathbf{x})  q(\mathbf{h}^2 | \mathbf{h}^1) \cdot \cdot  \cdot  q(\mathbf{h}^{L} | \mathbf{h}^{L-1}).
\end{align}
Evaluation of the posterior (recognition model)~\eqref{eq:posterior} involves forward information flows from the bottom of the tree to the top, and similarly, sampling  the generative model takes the reverse direction.

By leveraging the hierarchical conditional independence in both  generative model and  posterior,  the ELBO regarding the model is
\begin{align} \label{eq:elbo}
&\log p_{\theta}(\mathbf{x})
    \geqslant \mathcal{L}(\mathbf{x}; \theta)  = \mathbb{E}_{q(\mathbf{h}^{1:L}|\mathbf{x})}\big[ \log p(\mathbf{x}|\mathbf{h}^{1:L})  \big] - \sum_{l=1}^{L} \mathbf{KL}^l.
\end{align}
Here $\mathbf{KL}^l$ is the Kullback-Leibler divergence between the posterior and generative model in layer $l$. The first term in~(\ref{eq:elbo}) evaluates data reconstruction.
When $1\leqslant l \leqslant L$,
\begin{align}\label{eq:kl}
\mathbf{KL}^l
=\mathbb{E}_{q(\mathbf{h}^{1:L}|\mathbf{x})}\big[  \log q(\mathbf{h}^{l}|\mathbf{h}^{l-1})   - \log p(\mathbf{h}^{l}|\mathbf{h}^{l+1}) \big].
\end{align}
When $l=L$,
$\mathbf{KL}^L =  \mathbb{E}_{q(\mathbf{h}^{1:L}|\mathbf{x})}\big[  \log q(\mathbf{h}^{L}|\mathbf{h}^{L-1})- \log p(\mathbf{h}^{L})  \big].$ It is easy to extend the computation of the ELBO~(\ref{eq:elbo}) to DAGs with topology ordering of the nodes (and thus of the layers).
Let $ch(i)$ and $pa(i)$ denote node $i$'s child set and parent set, respectively.
Then, the ELBO for a DAG structure reads:
\begin{align}\label{eq:elbo_dag}
\mathcal{L}(\mathbf{x}; \theta) =& \mathbb{E}_{q(\mathbf{h}|\mathbf{x})}\big[ \log p(\mathbf{x}|\mathbf{h})  \big] -  \sum_{i \in \mathcal{V}  \setminus  \mathcal{R}_{ \mathbb{G} }} \textbf{\text{KL}}^{(i)}  -    \sum_{i \in  \mathcal{R}_{ \mathbb{G} }  }  \textbf{\text{KL}}\big(q(\mathbf{h}^{(i)} | \mathbf{h}^{ch(i)} )   || p(\mathbf{h}^{(i)})  \big) .
\end{align}
Here $\mathbf{KL}^{(i)}=\mathbb{E}_{q(\mathbf{h}|\mathbf{x})}\big[  \log q(\mathbf{h}^{(i)}|\mathbf{h}^{ch(i)})   - \log p(\mathbf{h}^{(i)}|\mathbf{h}^{pa(i)}) \big]$.  $\mathcal{R}_{ \mathbb{G}}$ is the set of root  nodes of DAG $\mathbb{G} = \{\mathcal{V}, \mathbf{f}\}$. Assuming there are $k$ leaf nodes on a tree or a DAG model, corresponding to $k$ sections of the input sample $\mathbf{x} = [\mathbf{x}^{(1)}, ..., \mathbf{x}^{(k)}]$.

Maximizing the ELBO~\eqref{eq:elbo} or~\eqref{eq:elbo_dag} equals to  optimizing the parameters of the flows, $\theta$.  Similar to VAEs, we apply forward message passing~(encoding) to approximate the posterior  distribution of each layer's latent variables, and backward message passing~(decoding) to  generate the reconstructions as shown in Figure~\ref{fig:tree_message}.
For the following sections, we use $\mathbf{h}^{i}$ to represent  node $i$'s state in the forward message, and $\widehat{\mathbf{h}}^{i}$ for  node $i$'s state in the backward message. For all nodes, both $\mathbf{h}^{i}$ and $\widehat{\mathbf{h}}^{i}$  are sampled from the posterior. At the rood nodes, we have  $\widehat{\mathbf{h}}^{\mathcal{R}}=\mathbf{h}^{\mathcal{R}}$ .

\subsection{ Aggregation Nodes}\label{sec:node_aggr}

 There are two approaches to aggregate signals from different nodes: average-based and concatenation-based. We rather focus on average-based aggregation in this paper,  and Figure~\ref{fig:dag_aggr} gives  an example denoted by the operator $\oplus$.
Let $\mathbf{f}_{(i, j)}$ be the direct edge~(function) from node $i$ to node $j$, and $\mathbf{f}^{-1}_{ (i, j)}$ or  $\mathbf{f}_{ (j, i)}$ defined as its inverse function. Then, the aggregation operation at  node $i$ reads
 \begin{align}~\label{eq:aggr_node}
&  \mathbf{h}^{(i)} = \frac{1}{|ch(i)|} \sum_{j \in ch(i) } \mathbf{f}_{(j,i)}(\mathbf{h}^{(j)}) , \quad  \widehat{\mathbf{h}}^{(i)} = \frac{1}{|pa(i)|} \sum_{j \in pa(i) } \mathbf{f}_{ (j,i)}(\widehat{\mathbf{h}}^{(j)}) \, .
\end{align}
Note that the above two equations hold even when node $i$ has only one child or parent.

\begin{figure}[h!]
\begin{center}
 \includegraphics[width=1.5in]{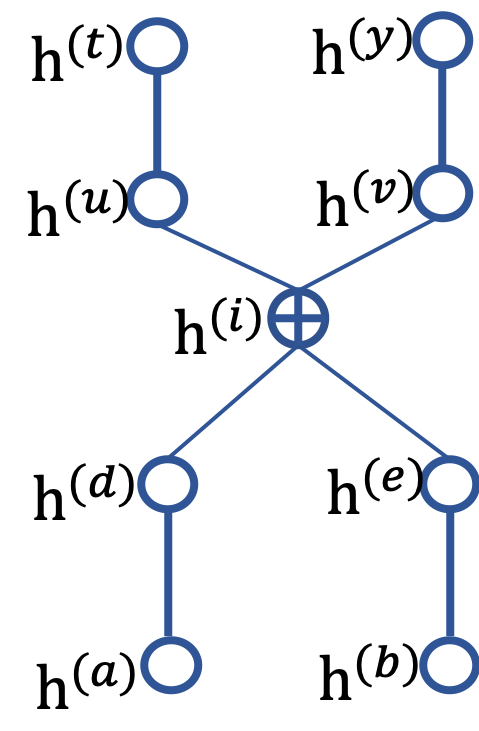}
\end{center}

\vspace{-0.1in}

\caption{Aggregation node on a DAG VFG.}
\label{fig:dag_aggr}
\end{figure}

With the identity function between the  parent and its children, there are \emph{node consistency rules} regarding an average aggregation node: {\it(a)} a  parent node's backward state  equals  the  mean of its children's forward states, i.e., $\widehat{\mathbf{h}}^{(i)} = \frac{1}{|ch(i)|} \sum_{j \in ch(i)} \mathbf{h}^{(j)}$; {\it(b)} a  child node's forward state equals to the average of its parents' backward states, i.e., $\mathbf{h}^{(i)} = \frac{1}{|pa(i)|} \sum_{j \in pa(i)} \widehat{\mathbf{h}}^{(j)} $. These rules empower VFGs with implicit invertibility.

We use aggregation node $i$ in the DAG presented in Figure~\ref{fig:dag_aggr} as an example to illustrate node consistency. Node $i$ has two parents, $u$ and $v$; and two children, $d$ and $e$. Node $i$  connects its parents and children with identity functions. According to~(\ref{eq:aggr_node}), we have $\mathbf{h}^{(i)} = (\mathbf{h}^{(d)}+\mathbf{h}^{(e)})/2$  and $\widehat{\mathbf{h}}^{(i)} = (\widehat{\mathbf{h}}^{(u)}+\widehat{\mathbf{h}}^{(v)})/2$.
Here aggregation \emph{consistency} means, for $i$'s children, their forward state should be consistent with $i$'s backward state, i.e.,
\begin{align}\label{eq:i_child}
\mathbf{h}^{(d)} = \mathbf{h}^{(e)} = \widehat{\mathbf{h}}^{(i)} .
\end{align}
For $i$'s parents, their backward state should be consistent with $i$'s forward state, i.e.,
\begin{align}\label{eq:i_parent}
\widehat{\mathbf{h}}^{(u)}  = \widehat{\mathbf{h}}^{(v)}  = \mathbf{h}^{(i)} .
\end{align}
We utilize the $\mathbf{KL}$ term in the ELBO~\eqref{eq:elbo_dag} to ensure~(\ref{eq:i_child}) and~(\ref{eq:i_parent}) can be satisfied during parameter updating. The $\mathbf{KL}$ term regarding~node~$i$~is
\begin{align}\label{eq:kl_dag}
\textbf{\text{KL}}^{(i)} = & \mathbb{E}_{q(\mathbf{h}, \widehat{\mathbf{h}}|\mathbf{x})}\big[  \log q(\mathbf{h}^{(i)}|\mathbf{h}^{ch(i)})  - \log p(\mathbf{h}^{(i)}|\widehat{\mathbf{h}}^{pa(i)}) \big]  \\ \notag
\simeq  & \log q(\mathbf{h}^{(i)}|\mathbf{h}^{ch(i)})  - \log p(\mathbf{h}^{(i)}|\widehat{\mathbf{h}}^{pa(i)}).
\end{align} 
As the term $\log q(\mathbf{h}^{(i)}|\mathbf{h}^{ch(i)})$  involves node states that are deterministic according to~\eqref{eq:aggr_node}, it is omitted in the computation of~\eqref{eq:kl_dag}. With Laplace as the latent state distribution, here
\begin{align} \notag
&\log p(\mathbf{h}^{(i)}|\widehat{\mathbf{h}}^{pa(i)}) \\ \notag
= &\frac{1}{2}\big(\log p(\mathbf{h}^{(i)}|\widehat{\mathbf{h}}^{(u)}) + p(\mathbf{h}^{(i)}|\widehat{\mathbf{h}}^{(v)})\big)\\ \notag
=& \frac{1}{2}\big(-\|\mathbf{h}^{(i)}- \widehat{\mathbf{h}}^{(u)}\|_1 -\|\mathbf{h}^{(i)}- \widehat{\mathbf{h}}^{(v)}\|_1-2m\cdot\log2 \big).
\end{align}
Hence minimizing $\textbf{\text{KL}}^{(i)}$ is equal to minimizing $\{\|\mathbf{h}^{(i)}- \widehat{\mathbf{h}}^{(u)}\|_1 + \|\mathbf{h}^{(i)}- \widehat{\mathbf{h}}^{(v)}\|_1 \}$ which achieves the consistent objective in~\eqref{eq:i_parent}.

Similarly,  $\textbf{\text{KL}}$s  of $i$'s children intend to realize consistency given in~\eqref{eq:i_child}. We use node $d$ as an example.  The $\textbf{\text{KL}}$ term regarding node $d$ is
\begin{align*}
\textbf{\text{KL}}^{(d)} = & \mathbb{E}_{q(\mathbf{h}, \widehat{\mathbf{h}}|\mathbf{x})}\big[  \log q(\mathbf{h}^{(d)}|\mathbf{h}^{ch(d)})  - \log p(\mathbf{h}^{(d)}|\widehat{\mathbf{h}}^{pa(d)}) \big] \\
\simeq  & \log q(\mathbf{h}^{(d)}|\mathbf{h}^{ch(d)})  - \log p(\mathbf{h}^{(d)}|\widehat{\mathbf{h}}^{pa(d)}).
\end{align*}
The first term $\log q(\mathbf{h}^{(d)}|\mathbf{h}^{ch(d)})$ is omitted in the calculation of  $\textbf{\text{KL}}^{(d)}$ due to the deterministic relation with~\eqref{eq:aggr_node}. Knowing that
\begin{align*}
 \log p(\mathbf{h}^{(d)}|\widehat{\mathbf{h}}^{pa(d)})= & \log p(\mathbf{h}^{(d)}|\widehat{\mathbf{h}}^{(i)}) \\
 = & -\|\mathbf{h}^{(d)}- \widehat{\mathbf{h}}^{(i)}\|_1 - m\cdot\log2,
\end{align*} 
we notice that minimizing $\textbf{\text{KL}}^{(d)}$ boils down to minimizing $\|\mathbf{h}^{(d)}- \widehat{\mathbf{h}}^{(i)}\|_1$ that targets at \eqref{eq:i_child}.
In summary, by maximizing the ELBO of a VFG, the aggregation consistency can be  attained along with fitting the model to the data.

\subsection{Implementation Details}

The calculation of the data reconstruction term in~\eqref{eq:elbo_dag} requires  node states $\mathbf{h}^{i}$ and $\widehat{\mathbf{h}}^{i}$ ($\forall i \in \mathcal{V}$) from the posterior. They correspond to the encoding and decoding procedures in VAE model as shown in Eq.~(\ref{eq:vae_recon}). At the root node,  we have $\widehat{\mathbf{h}}^{\mathcal{R}}=\mathbf{h}^{\mathcal{R}} $.  The reconstruction terms in ELBO~(\ref{eq:elbo_dag}) can be computed with the backward message in the generative model $p(\mathbf{x}| \widehat{\mathbf{h}}^{1})$, i.e.,
\begin{align*}
&\mathbb{E}_{q(\mathbf{h}, \widehat{\mathbf{h}}|\mathbf{x})}\big[ \log p(\mathbf{x}|\mathbf{h}, \widehat{\mathbf{h}})\big]
\simeq  \frac{1}{U}\sum_{u=1}^U \log p(\mathbf{x}| \widehat{\mathbf{h}}^{1:L}_u)
= \frac{1}{U}\sum_{u=1}^U \log p(\mathbf{x}| \widehat{\mathbf{h}}^{pa(x)}_u) .
 \end{align*}
For a VFG model, we set $U=1$. In the last term,  $p(\mathbf{x}| \widehat{\mathbf{h}}^{pa(x)})$ is either Gaussian or binary distribution parameterized with $\widehat{\mathbf{x}}$ generated via the flow function with $\widehat{\mathbf{h}}^{pa(x)}$ as the input.

\section{Universal Approximation Property}\label{sec:approx}
A universal approximation power of coupling-layer based flows has been highlighted in~\cite{teshima2020coupling}.
Following the analysis for flows~\cite{teshima2020coupling}, we  prove that coupling-layer based VFGs have universal approximation as well.   We first  give several additional definitions regarding universal approximation.
For a measurable mapping $\mathbf{f}: \mathbb{R}^m \rightarrow  \mathbb{R}^n$ and a subset $K \subset   \mathbb{R}^m $, we define the following,
\begin{align*}
|| \mathbf{f}||_{p,K}= \bigg(\int_{K} ||f(x)||^p dx \bigg)^{1/p}.
\end{align*}
Here $||\cdot||$ is the Euclidean norm of $\mathbb{R}^n$ and $||\mathbf{f}||_{\text{sup},K} := \text{sup}_{x\in K} || \mathbf{f}(x)||$.

\begin{definition}
($L^p$-/sup-universality) Let $\mathcal{M}$ be a model which is a set of measurable mappings from $\mathbb{R}^m$ to $\mathbb{R}^n$. Let $p\in [1, \infty)$, and let $\mathcal{G}$ be  a set of measurable mappings $\mathbf{g}: U_{\mathbf{g}} \rightarrow \mathcal{R}^n$, where $U_{\mathbf{g}}$ is a measurable subset of $\mathbb{R}^m$ which may depend on $\mathbf{g}$. We say that $\mathcal{M}$ has the $L^p$-universal approximation property for $\mathcal{G}$ if for any $\mathbf{g}\in \mathcal{G}$, any $\epsilon > 0$, and any compact subset $K \in U_{\mathbf{g}}$, there exists  $\mathbf{f} \in \mathcal{M}$ such that $|| \mathbf{f}-\mathbf{g}||_{p, K} < \epsilon$. We define the sup-universality analogously by replacing  $|| \cdot ||_{p, K}||$ with $|| \cdot ||_{p, K}||_{sup, K}$ .
\end{definition}

\begin{definition}
(Immersion and submanifold) $\mathbf{g}:\mathfrak{M}  \rightarrow \mathfrak{N} $ is said to be an immersion if rank($\mathbf{g}$)$=m=$dim($\mathfrak{M}$) everywhere. If $\mathbf{g}$ is injective~(one-to-one) immersion, then $\mathbf{g}$ establish an one-to-one correspondence of  $\mathfrak{M}$ and the subset $\Tilde{\mathfrak{M}} = \mathbf{g}(\mathfrak{M})$ of $\mathfrak{N}$. If we use this correspondence to endow $\Tilde{\mathfrak{M}}$ with a topology and $\mathcal{C}^{\infty}$ structure, then $\Tilde{\mathfrak{M}}$ will be called a submanifold (or immersed submanifold) and $\mathbf{g}:\mathfrak{M} \rightarrow \Tilde{\mathfrak{M}}$ is a diffeomorphism.
\end{definition}

\begin{definition}
($\mathcal{C}^r$-diffeomorphisms for submanifold: $\mathcal{Q}^r$). We define $\mathcal{Q}^r$  as the set of all $\mathcal{C}^r$-diffeomorphisms $\mathbf{g}: U_\mathbf{g} \rightarrow \mathfrak{U} $, where $U_\mathbf{g} \subset \mathbb{R}^m$ is an open set $\mathcal{C}^r$-diffeomorphic to $\mathfrak{U}$, which may depend on $\mathbf{g}$, and $\mathfrak{U}$ is a submanifold of $\mathbb{R}^n$.
\end{definition}

We use $m$ to represent the root node dimension of a VFG, and $n$ to denote the dimension of data samples. VFGs learn the data manifold embedded in $\mathbb{R}^n$. We define $\mathcal{C}_c^{\infty}(\mathbb{R}^{m-1})$ as the set of all compactly-supported $\mathcal{C}^{\infty}$ mappings from $\mathbb{R}^{m-1}$ to $\mathbb{R}$. For a function set $\mathcal{T}$, we define $\mathcal{T}$-ACF as the set of affine coupling flows~\cite{teshima2020coupling} that  are assembled  with functions in $\mathcal{T}$, and  we use VFG$_{\mathcal{T}-ACF}$ to represent the set of VFGs constructed using flows in $\mathcal{T}$-ACF.

\vspace{0.1in}

\begin{theorem} \label{thm:lp_univ}
($L^p$-universality) Let $p \in [0, \infty)$ . Assume $\mathcal{H}$ is a sup-universal approximator for $\mathcal{C}_c^{\infty}(\mathbb{R}^{m-1})$, and that it consists of $\mathcal{C}^1$-functions. Then  VFG$_{\mathcal{H}-ACF}$ is an $L^p$-universal approximator~for~$\mathcal{Q}_c^0$.
\end{theorem}
\begin{proof}
We  construct a VFG structure that forms a mapping from $\mathbb{R}^m$ to $\mathbb{R}^n$.  Let $r=n\mod m$.

If  $r =0$, it is easy to construct a one-layer tree VFG $\mathbf{f}$ ($\mathbf{f}$ also represents the function/edge set) and the root as an aggregation node. The children divide the $n$ input entries into $\tau = n/m$ even sections, and each section  connects the aggregation node with a flow function.

Given an injective immersion $\mathbf{g}:\mathfrak{M}  \rightarrow \mathfrak{N}$, function $\mathbf{g}$ can be represented with the concatenation of a set of functions, i.e.,  $\mathbf{g}=[\mathbf{g}_1, ...,  \mathbf{g}_{\tau}]^{\top}$, each invertible $\mathbf{g}_i$ has dimension $m$. According to the function decomposition theory~\cite{kuo2010decompositions}, its inverse  can be represent as the summation of  functions $\mathbf{g}^{-1}_{i}, 1\leq i \leq \tau$, i.e., $\mathbf{g}^{-1} = \frac{1}{\tau} \sum_{i=1}^{\tau} \mathbf{g}^{-1}_{i}$. For each $\mathbf{g}_{i}$, and $\tilde{\mathfrak{M}}_{i} = \mathbf{g}_{i}(\mathfrak{M})$ is a submanifold in $\mathfrak{N}$, and it is diffeomorphic to $\mathfrak{M}$. According to  Theorem 2 in~\cite{teshima2020coupling}, $\mathcal{H}-ACF$ is  an  universal approximater  for each $\mathbf{g}_{i}$, $1\leq i \leq \tau$. Therefore,  VFG $\mathbf{f}$ has  universal approximation for immersion $\mathbf{g}:\mathfrak{M}  \rightarrow \mathfrak{N}$.

If $r \neq 0$, let $\tau = \lfloor n/m \rfloor$. We  divide the  $\tau$-th section and the remaining $r$ entries into two equal small sections that are denoted with $\tau$ and $\tau +1$. Sections $\tau$ and $\tau +1$ have  $r$  overlapped entries. Similarly, we can construct an one-layer VFG $\mathbf{f}$ with $\tau +1$ children,  and  each child takes a section as the input.

The input coordinate index of $\mathbf{g}_{\tau}$ in $\mathbb{R}^m$ is $I_{\tau} = \big[1,2,..., \lceil (m+r)/2 \rceil \big]$, and the output index of $\mathbf{g}_{\tau}$  in $\mathbb{R}^n$  is $I_{\tau} + \gamma = \big[\gamma + 1, \gamma + 2,...,  \gamma + \lceil (m+r)/2 \rceil \big]$, and $\gamma = (\tau-1)m$. The input coordinate index of $\mathbf{g}_{\tau+1}$ in $\mathbb{R}^m$ is $I_{\tau + 1} = \big[m- \lceil (m+r)/2 \rceil + 1, ..., m-1, m \big]$, and the output index of $\mathbf{g}_{\tau +1}$  in $\mathbb{R}^n$  is $I_{\tau+1} + \gamma $.   We can see that the m dimensions are divided into two sets, the overlapped set $O = \big[ m- \lceil (m+r)/2 \rceil + 1, \lceil (m+r)/2 \rceil  \big]$, and the remaining set $R$ containing the rest dimensions.

The mapping $\mathbf{g}:\mathfrak{M}  \rightarrow \mathfrak{N}$ can be decomposed into $\tau + 1$ functions, i.e.,  $\mathbf{g}=[\mathbf{g}_1, ..., \mathbf{g}_{\tau}, \mathbf{g}_{\tau+1}]^{\top}$, and the inverse $\mathbf{g}^{-1}$ is adjusted here:  $\mathbf{g}_j^{-1} = \frac{1}{\omega} \sum_{i=1}^{\omega} \mathbf{g}^{-1}_{i(j)}$. When $j\in O$, $\omega = \tau +1$, and all $\mathbf{g}^{-1}_i$s will be involved; when $j\in R$, $\omega = \tau$, and either $\mathbf{g}^{-1}_{\tau}$  or $\mathbf{g}^{-1}_{\tau +1}$ is omitted due to the missing of  entry $j$ in the function output. The mapping  $\mathbf{g}_{\tau}$  is a diffeomorphism from manifold  $\mathfrak{M}_{\tau}$ ( $\mathfrak{M}_{\tau} \subset \mathfrak{M}$) to  sub-manifold $\tilde{\mathfrak{M}}_{\tau}$  in $\mathfrak{N}$. Similarly
$\mathbf{g}_{\tau+1}$ is a  diffeomorphism from  $\mathfrak{M}_{\tau+1}$  to  manifold $\tilde{\mathfrak{M}}_{\tau+1}$.  For each $\mathbf{g}_{i}$, $1\leq i \leq \tau +1$, it can be universally approximated with a function in $\mathcal{H}-ACF$~\cite{teshima2020coupling}. Hence, we construct a VFG with universal approximation~for~any~$\mathbf{g}$~in~$\mathcal{Q}_c^0$.
\end{proof}

\vspace{0.1in}

With the conditions in Theorem~\ref{thm:lp_univ},  VFG$_{\mathcal{H}-ACF}$ is a distributional universal approximator as well~\citep{teshima2020coupling}.

\section{The Proposed Algorithms}\label{sec:algrithm}

In this section, we develop the training algorithm (Algorithm~\ref{alg:main}) to  maximize  the ELBO objective function \eqref{eq:elbo_dag}.
In Algorithm~\ref{alg:main}, the inference of the latent states is performed via forwarding message passing, cf. Line~6, and their reconstructions are computed in backward message passing, cf. Line~11.
A VFG is a deterministic network passing latent variable values between nodes. Ignoring explicit neural network parameterized variances for all latent nodes enables us to use flow-based models as both the encoders and decoders.
Hence, we obtain a deterministic ELBO objective~(\ref{eq:elbo})-~(\ref{eq:elbo_dag}) that can efficiently be optimized with standard stochastic optimizers.
\begin{algorithm}[h]
  \caption{Inference model parameters with  forward and backward message propagation}
   \label{alg:main}
\begin{algorithmic}[1]
   \STATE {\bfseries Input:} Data distribution $\mathcal{D}$,  $\mathbb{G} = \{\mathcal{V}, \mathbf{f}\}$
   \FOR {$s=0,1,...$}
   \STATE  Sample minibatch $b$ samples $\{\mathbf{x}_1, ..., \mathbf{x}_b \}$ from $\mathcal{D}$;
   \FOR{$i \in \mathcal{V}$}\label{line:for2}
    \STATE  \textcolor{blue}{// forward message passing}
   \STATE $\mathbf{h}^{(i)} = \frac{1}{|ch(i)|} \sum_{j \in ch(i) } \mathbf{f}_{(j,i)}(\mathbf{h}^{(j)})$; \label{line:forward}
    \ENDFOR
    \STATE $\widehat{\mathbf{h}}^{(i)} = \mathbf{h}^{(i)} \ \  \text{if} \ i \in \mathcal{R}_{\mathbb{G}} $ or $i \in$ layer L;
   \FOR{$i \in \mathcal{V}$}
   \STATE \textcolor{blue}{// backward message passing}
   \STATE $\widehat{\mathbf{h}}^{(i)} = \frac{1}{|pa(i)|} \sum_{j \in pa(i) } \mathbf{f}^{-1}_{ (i,j)}(\widehat{\mathbf{h}}^{(j)}) $;\label{line:backward}
   \ENDFOR
    \STATE  $\mathbf{h} =  \{\mathbf{h}^{(t)} \big |  t \in \mathcal{V} \}$, $\widehat{\mathbf{h}} =  \{\widehat{\mathbf{h}}^{(t)} \big | t \in \mathcal{V} \}$;
    \STATE Approximate the $\mathbf{KL}$ terms in ELBO for each layer with b samples;
    \STATE Updating VFG model $\mathbb{G}$ with gradient ascending: $\theta^{(s+1)}_{\mathbf{f}} = \theta^{(s)}_{\mathbf{f}} + \nabla_{\theta_{\mathbf{f}}}\frac{1}{b} \sum_{i=1}^b  \mathcal{L}(\mathbf{x}_b; \theta^{(s)}_{\mathbf{f}})   \, .$\label{line:update}
   \ENDFOR
\end{algorithmic}
\end{algorithm}

In training Algorithm~\ref{alg:main}, the backward variable state $\widehat{\mathbf{h}}^l$ in  layer $l$  is generated according to $p(\widehat{\mathbf{h}}^l | \widehat{\mathbf{h}}^{l+1})$, and at the root  layer,  node state $\widehat{\mathbf{h}}^{\mathcal{R}}$ is set  equal to  $\mathbf{h}^{\mathcal{R}}$ that is
from  the posterior $q(\mathbf{h}|\mathbf{x})$, not from the prior $p(\mathbf{h}^{\mathcal{R}})$. So we can see all the forward and backward latent variables are sampled from the posterior $q(\mathbf{h}|\mathbf{x})$.


\newpage

From a practical perspective, layer-wise training strategy can improve the accuracy of a model especially when it is constructed of more than two layers.
In such a case, the parameters of only one layer are updated with backpropagation of the gradient of the loss function while keeping the other layers fixed at each optimization step.
By maximizing the ELBO~\eqref{eq:elbo_dag} with the above algorithm, the node consistency rules in Section~\ref{sec:node_aggr} are expected to be satisfied.

\subsection{Improve Training of VFG}\label{sec:random_mask}

The inference ability of VFG  can be reinforced by masking out some sections of the training samples.
The training objective can be changed to force the model to impute the value of the masked sections.
For example in a tree model, the alternative objective function reads
\begin{align}  \label{eq:elbo_tree_mask}
 \mathcal{L}(\mathbf{x}, O_{\mathbf{x}}; \theta)
= & \sum_{t: 1\leqslant t \leqslant k, t\notin O}
 \mathbb{E}_{q(\mathbf{h}, \widehat{\mathbf{h}}|\mathbf{x}^{O_{\mathbf{x}}} )} \bigg[ \log p( \mathbf{x}^{(t)}|  \widehat{\mathbf{h}}^{1})   \bigg] \\ \notag
 &- \sum_{l=1}^{L-1}  \mathbb{E}_{q(\mathbf{h}, \widehat{\mathbf{h}}|\mathbf{x})} \bigg[ \log q(\mathbf{h}^{l}|\mathbf{h}^{l-1}) - \log p( \mathbf{h}^{l}|  \widehat{\mathbf{h}}^{l+1})   \bigg]
 \\ \notag
 &-  \textbf{\text{KL}}\big(q(\mathbf{h}^L | \mathbf{h}^{L-1} )   | p(\mathbf{h}^L)  \big).
\end{align}
where $O_{\mathbf{x}}$ is the index set of leaf nodes  with observation, and $\mathbf{x}^{O_{\mathbf{x}}}$ is the union of observed data sections.
The  random-masking training procedure for objective~\eqref{eq:elbo_tree_mask} is described in Algorithm~\ref{alg:rand_mask}.
In practice, we use Algorithm~\ref{alg:rand_mask} along with Algorithm 1 to enhance the training of a VFG model.  However, we only occasionally update the model parameter $\theta$ with the gradient of~\eqref{eq:elbo_tree_mask} to ensure the distribution learning running well.

\begin{algorithm}[b!]
   \caption{Inference model parameters with random masking}
   \label{alg:rand_mask}
\begin{algorithmic}[1]
   \STATE {\bfseries Input:} Data distribution $\mathcal{D}$,  $\mathbb{G} = \{\mathcal{V}, \mathbf{f}\}$
   \FOR {$s=0,1,...$}
   \STATE  Sample minibatch $b$ samples $\{\mathbf{x}_1, ..., \mathbf{x}_b \}$ from $\mathcal{D}$;
   \STATE
    Optimize~(\ref{eq:elbo}) with Line~4 to Line~15 in Algorithm~\ref{alg:main};
    \STATE  Sample a subset of the $k$ data sections as data observation set $O_{\mathbf{x}}$; $O \leftarrow O_{\mathbf{x}}$;
   \FOR{$i \in \mathcal{V}$}
    \STATE  \textcolor{blue}{// forward message passing}
   \STATE $\mathbf{h}^{(i)} = \frac{1}{|ch(i) \cap O |} \sum_{j \in ch(i) \cap O} \mathbf{f}_{(j,i)}(\mathbf{h}^{(j)})$;
     \STATE  $O \leftarrow O \cup \{i\}$ if $ch(i) \cap O \neq \emptyset $;
    \ENDFOR
    \STATE $\widehat{\mathbf{h}}^{(i)} = \mathbf{h}^{(i)} \ \  \text{if} \ i \in \mathcal{R}_{\mathbb{G}} $ or $i \in$ layer L;
   \FOR{$i \in \mathcal{V}$}
   \STATE \textcolor{blue}{// backward message passing}
   \STATE $\widehat{\mathbf{h}}^{(i)} = \frac{1}{|pa(i)|} \sum_{j \in pa(i) } \mathbf{f}^{-1}_{ (i,j)}(\widehat{\mathbf{h}}^{(j)}) $;
   \ENDFOR
    \STATE  $\mathbf{h} =  \{\mathbf{h}^{(t)} \big |  t \in \mathcal{V} \cap O \}$, $\widehat{\mathbf{h}} =  \{\widehat{\mathbf{h}}^{(t)} \big | t \in \mathcal{V} \}$;
    \STATE Approximate the $\mathbf{KL}$ terms in ELBO for each layer with b samples;
    \STATE Updating VFG with gradient of~(\ref{eq:elbo_tree_mask}): $\theta^{(s+1)}_{\mathbf{f}} = \theta^{(s)}_{\mathbf{f}} + \nabla_{\theta_{\mathbf{f}}}\frac{1}{b} \sum_{i=1}^b  \mathcal{L}(\mathbf{x}_b, O_{\mathbf{x}}; \theta^{(s)}_{\mathbf{f}})   \, ,$
   \ENDFOR
\end{algorithmic}
\end{algorithm}

\newpage

\begin{figure}[t]
\begin{center}
 \includegraphics[width=0.23\linewidth]{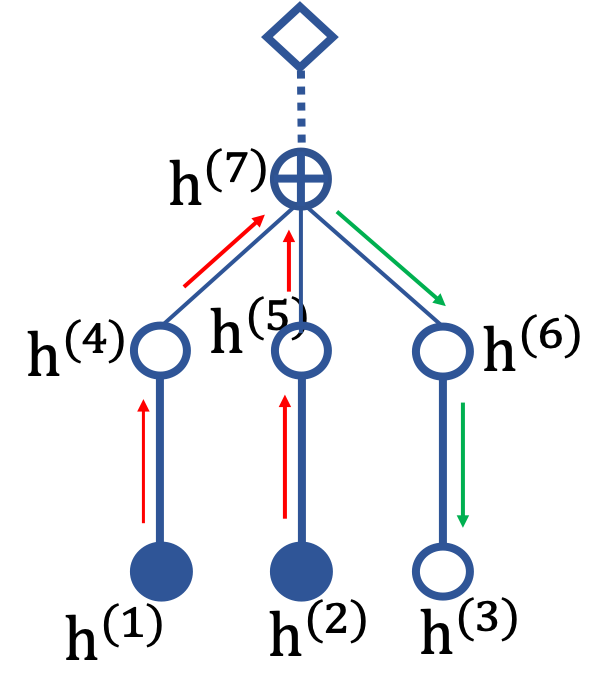}
 \hspace{0.4in}
 \includegraphics[width=0.49\linewidth]{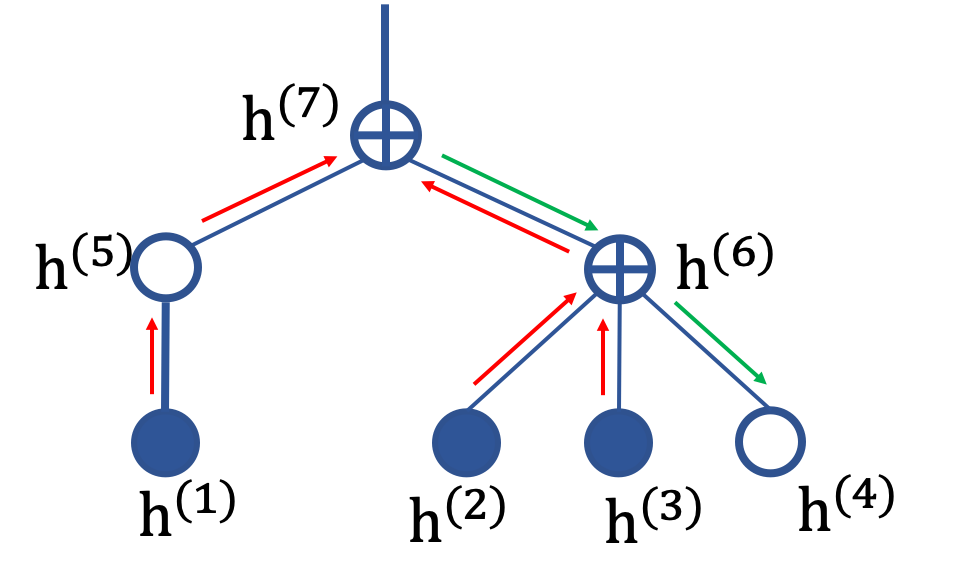}
\end{center}
\vspace{-0.1in}
 \caption{{(Left) Inference on a VFG with single aggregation node. Node 7 aggregates information from node 1 and 2, and  passes down the update to node 3 for prediction. (Right) Inference on a tree VFG. Observed node states are gathered at node 7 to predict the state of node 4. Red and green lines are forward and backward messages, respectively.}}
\label{fig:two_layer_infer}
\end{figure}

\section{Inference on VFGs }\label{sec:infer}

With a  VFG, we aim to infer node states given observed ones. The hidden state of a parent node $j$ in $l=1$ can be computed with the observed children as follows:
 \begin{align}\label{eq:aggr_obs_ch}
\mathbf{h}^{(j)}  = \frac{1}{|ch(j) \cap O|}\sum_{i \in ch(j) \cap O} \mathbf{h}^{(i)} \, ,
\end{align}
where $O$ is the set of observed leaf nodes, see Figure~\ref{fig:two_layer_infer}-left for an illustration.
Observe that for either a tree or a DAG, the state of any hidden node is updated via messages received from its children. After reaching the root node, we can update any nodes with backward message passing.  Figure~\ref{fig:two_layer_infer} illustrates this inference mechanism for trees in which the structure enables us to perform message passing among the nodes.
We derive the following lemma establishing the relation between two leaf nodes.

\begin{lemma}\label{lm:apprx}
Let $\mathbb{G}$ be a  tree VFG  with $L$ layers, and $i$ and $j$ are two leaf nodes with $a$ as the closest common ancestor node. Given observed value at node $i$, the value of node $j$ can be approximated by   $\widehat{\mathbf{x}}^{j} =  \mathbf{f}_{(a,j)}(\mathbf{f}_{(i, a)}(\mathbf{x}^{(i)}))$. Here $\mathbf{f}_{(i, a)}$ is the flow function path from node $i$ to node $a$.
\end{lemma}
\begin{proof}
According to the  aggregation operation~\eqref{eq:aggr_node} discussed in Section~\ref{sec:node_aggr}, at an aggregation node  $a$, the reconstruction  state of a child node $j$ is the mean reconstruction state averaging the backward messages from the parent nodes.
The reconstruction of the child node $j$ can be calculated with the average reconstruction state regarding its parent node. Apply it sequentially,  we have $\widehat{\mathbf{x}}^{(j)} = \mathbf{f}_{(a,j)}(\widehat{\mathbf{h}}^{a)})$. The forward state of node $a$ can be computed by sequentially applying forward aggregating starting from its observed descendent $i$, i.e., $\mathbf{h}^{(a)} = \mathbf{f}_{(i,a)}(\mathbf{x}^{(i)})$.
As there are no other observations, with forward and backward message passing to and from the root node, at  node $a$, we have $\mathbf{h}^{(a)} = \widehat{\mathbf{h}}^{(a)}$.
Therefore, we have $\widehat{\mathbf{x}}^{(j)} =  \mathbf{f}_{(a,j)}(\mathbf{f}_{(i, a)}(\mathbf{x}^{(i)}))$.
\end{proof}

\newpage

Considering the flow-based model~(\ref{eq:flow}), we have the following identity for each node of the graph structure:
\begin{align*}
 p(\mathbf{h}^{(i)} | \mathbf{h}^{pa(i)})  &= p(\mathbf{h}^{pa(i)}) \big|\det(\frac{\partial \mathbf{h}^{pa(i)} }{\partial \mathbf{h}^{(i)}})\big| \\
 &= p(\mathbf{h}^{pa(i)}) \big|\det(\mathbf{J}_{\mathbf{h}^{pa(i)}}(\mathbf{h}^{(i)}))\big| \, .
\end{align*}
Lemma~\ref{lm:apprx} provides an approach to conduct inference on a tree and impute missing values in the data. It is easy to extend the inference method to DAG VFGs.

\vspace{0.1in}
\section{Numerical Experiments}\label{sec:numerical}

In this section, we provide several studies to validate the proposed VFG models.
The first  application we present is missing value imputation. We compare our method with different baseline models  on several  datasets.
The second set of experiments is to evaluate VFG models on three different datasets, i.e.,  MNIST, Caltech101, and Omniglot, with ELBO and likelihoods as the score.
The third application we present here is the task of learning posterior distribution of the latent variables corresponding to the hidden explanatory factors of variations in the data~\citep{bengio2013representation}.
For that latter application, the model is trained and evaluated on the MNIST handwritten digits dataset.

\vspace{0.1in}

In this paper,  we would rather assume the VFG graph structures are given and fixed.  In the following experiments, the VFG structures are given in the dataset or  designed heuristically (as other neural networks) for the sake of numerical illustrations.  Learning the structure of VFG is an interesting research problem and is left for future works. A simple approach for VFG structure learning is to regularize the graph with the DAG structure penalty~\citep{Zheng2018,wehenkel2021graphical}.

All the experiments are conducted on NVIDIA-TITAN X (Pascal) GPUs.
In  the experiments, we use the same  coupling block~\citep{Dinh2016DensityEU} to construct different flow functions. The coupling block consists of three fully connected layers~(of dimension $64$) separated by two RELU layers along with the coupling trick.
Each flow function has block number $\mathcal{B} > 3$.


\subsection{Evaluation on Inference with Missing Entries Imputation}

We now focus on the task of imputing missing entries in a graph structure. For all the following experiments, the models are trained on the training set and are used to infer the missing entries of samples in the testing set.
We first study the proposed VFGs on two datasets without given graph structures, and we compare VFGs with several conventional methods that do not require the  graph  structures in the data.  We then compare VFGs with graphical models that can perform inference on explicit graphs.

\subsubsection{Synthetic Dataset}

In this set of experiments, we study different methods  with synthetic datasets. The baselines for this set of experiments include mean value method~(Means), iterative imputation~(Iterative)~\citep{buck1960method}, and multivariate imputation by chained equation~(MICE)~\citep{van2011mice}.  Mean Squared Error as the metric of reference in order to compare the different methods for the imputation task. We use the baseline implementations in~\citet{scikit-learn}  in the experiments.

\vspace{0.1in}

We generate $10$ synthetic datasets (using different seeds) of $1,300$ data points, $1,000$ for the training phase of the model, $300$ for imputation testing.
Each data sample  has $8$ dimensions with $2$ latent variables.
Let $z_1 \sim \mathcal{N}(0,1.0^2)$ and $z_2 \sim  \mathcal{N}(1.0,2.0^2)$ be the latent variables. For a sample $\mathbf{x}$, we have  $x_1=x_2 = z_1, x_3=x_4= 2\textrm{sin}(z_1), x_5=x_6 =z_2$, and $x_7= x_8 = z_2^2$.  In the testing dataset, $x_3$, $x_4$, $x_7$, and $x_8$ are missing. We use a VFG model with a single average aggregation node that has four children, and each child connects the parent with a flow function consisting of 3 coupling layers~\citep{Dinh2016DensityEU}. Each child takes 2 variables as input data section, and the latent dimension of the VFG is $2$.
We compare, in Figure~\ref{fig:sim}, our VFG method with the baselines described above using boxplots on obtained MSE values for those $10$ simulated datasets.
We can see that the proposed VFG model performs much better than mean value, iterative, and MICE methods. Figure~\ref{fig:sim} shows that VFGs also demonstrates more performance robustness compared against other methods.

\begin{figure}[h]

\vspace{-0.1in}

  \centering
      \includegraphics[width=2.5in]{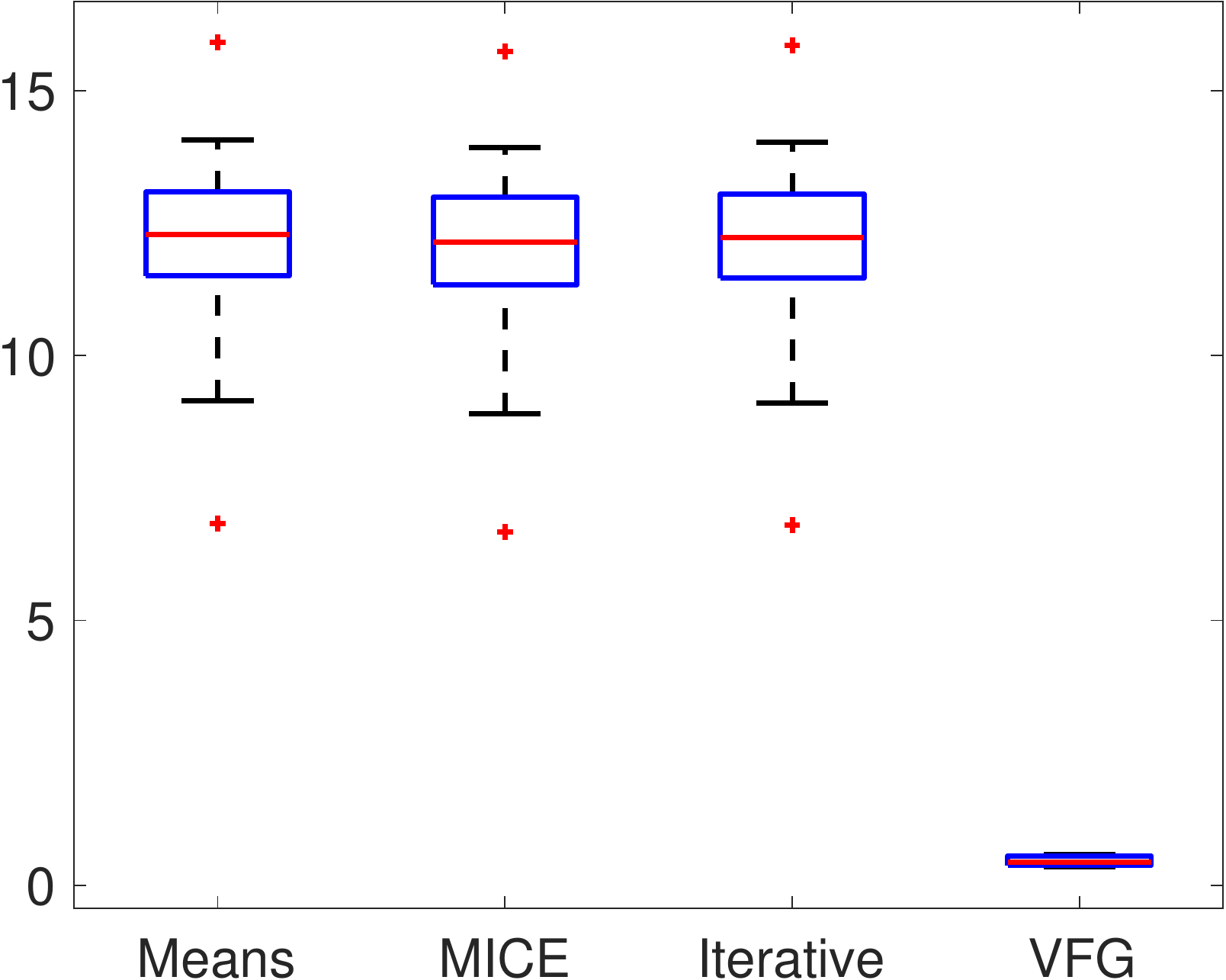}

\vspace{-0.1in}

    \caption{Synthetic datasets: MSE boxplots of VFG and baseline methods.}
    \label{fig:sim}\vspace{-0.15in}
\end{figure}

\subsubsection{California Housing Dataset}
We further investigate the method on a real dataset.
The California Housing dataset  has 8 feature entries and $20,640$ data samples.
We use the first $20,000$ samples for training  and $100$ of the rest for testing.
We get  4 data sections, and each section contains 2 variables.
In the testing set, the second section is assumed missing for illustration purposes, as the goal is to impute this missing section. In addition to the three baselines in introduced the main file, we also compared with KNN~(k-nearest neighbor) method. Again,  we use the  implementations from~\citet{scikit-learn}  for the baselines in this set of  experiments.

\begin{table}[b!]

\vspace{-0.2in}

\centering
\caption{California Housing dataset: Imputation Mean Squared Error (MSE) results.}\vspace{0.1in} \label{tab:imp_arrhytmia}
 \begin{tabular}{l | c  }\hline
\textit{Methods} & \textit{Imputation MSE}  \\
\hline
Mean Value &1.993 \\
MICE & 1.951\\
Iterative Imputation & 1.966\\
KNN (k=5) &1.969 \\
\hline
VFG & \textbf{1.356} \\
\hline
\end{tabular}\vspace{-0.1in}
\end{table}

The VFG structure is designed heuristically.
We construct a tree structure VFG with 2 layers. The first layer has two aggregation nodes, and each of them has two children.
The second layer consists of one aggregation node that has two children connecting with the first layer.
Each flow function has $\mathcal{B}=4$ coupling blocks.
 Table~\ref{tab:imp_arrhytmia} shows that our model yields significantly better results than any other method in terms of prediction error. It indicates that with the help of universal approximation power of neural networks, VFGs have superior inference capability.

\subsubsection{Comparison with Graphical Models}

In this set of experiments, we use a synthetic Gaussian graphical model dataset from the bnlearn package~\citep{scutari2009learning} to evaluate the proposed model. The data graph structure is given. The dataset consists of 7 variables and 5,000 samples.  Sample values at each node are generated according to a  structured causal model with  a  diagram given by Figure~\ref{fig:gaussian_graph}.

\begin{figure}[h!]

\vspace{0.2in}

\begin{center}
 \includegraphics[width=2.5in]{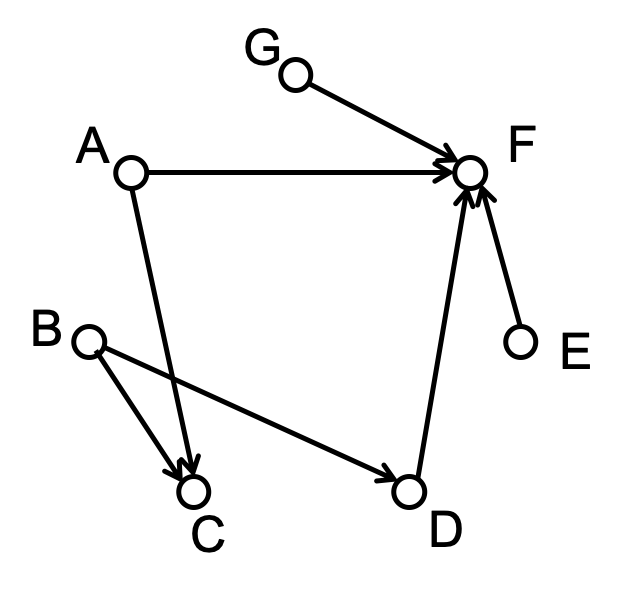}
\end{center}

\vspace{-0.4in}

\caption{Graph structure for Gaussian graphical model dataset. }
\label{fig:gaussian_graph}\vspace{0.2in}
\end{figure}

In Figure~\ref{fig:gaussian_graph}, each node represents a variable generated with a function of its parent nodes. For instance, node $V$ is generated with  $V= \mathbf{f}(pa(V), N_V)$. Here $pa(V)$ is the set of $V$'s parents, and $N_V$ is a noise term for $V$.    A node without any parent is determined only by the noise term.  $\mathbf{f}()$ is $V$'s generating function, and  only linear functions are used in this dataset. All the noise terms are Normal distributions.

\vspace{0.1in}

We take  Bayesian network implementation~\citep{scutari2009learning}  and sum-product network (SPN) package~\citep{SPFlow,poon2011sum} as  experimental baselines.  4\,500 samples are used for training, and the rest 500 samples  are for testing. The structure of VFG is designed based on  the directed graph given by Figure~\ref{fig:gaussian_graph}. In the imputation task, we take Node `F' as the missing entry, and use the values of other node to impute the missing entry.  Table~\ref{tab:gaussian} gives the imputation results from the three methods.  We can see that VFG achieves the smallest prediction error. Besides the imputation MSE,  Table~\ref{tab:gaussian} also gives the prediction error variance. Compared against  Bayesian net  and SPN, VFG achieves much smaller performance variance. It means VFGs are much more stable in this set of experiments.

\begin{table}[h]
\centering
\caption{Gaussian graphical model dataset: Imputation Mean Squared Error (MSE) and Variance results.} \label{tab:gaussian}\vspace{0.1in}
 \begin{tabular}{l | c | c | c  }\hline
\textit{Methods} & Bayesian Net & SPN & VFG \\
\hline
Imputation MSE &  1.059 & 0.402 &  \textbf{0.104} \\
Imputation Variance  & 2.171  &   0.401 &  \textbf{0.012} \\
\hline
\end{tabular}
\end{table}

\begin{table}[b!]

\vspace{-0.1in}

\caption{Numerical values of negative log-likelihood and free energy (negative evidence lower bound) for static MNIST, Caltech101, and Omniglot datasets.}\vspace{0.1in}
\centering
\label{tab:elbo}
\resizebox{1.0\columnwidth}{!}{
\begin{tabular}{l | c  c   c  c  c  c }
\hline
 \multirow{2}{0nc}{\textbf{Model}} & \multicolumn{2}{c}{\textbf{MNIST}} & \multicolumn{2}{c}{\textbf{Caltech101}} & \multicolumn{2}{c}{\textbf{Omniglot}} \\
 & -ELBO & NLL  &  -ELBO & NLL  & -ELBO & NLL  \\
\hline
 VAE~\citep{kingma2013auto} & 86.55 $\pm$ 0.06  & 82.14 $\pm$ 0.07& 110.80 $\pm$ 0.46 & 99.62 $\pm$ 0.74 & 104.28 $\pm$ 0.39 & 97.25 $\pm$ 0.23 \\
Planer~\citep{rezende2015variational} & 86.06 $\pm$ 0.31 & 81.91 $\pm$ 0.22 & 109.66 $\pm$ 0.42 & 98.53 $\pm$ 0.68 & 102.65 $\pm$ 0.42 & 96.04 $\pm$ 0.28 \\
IAF~\citep{kingma2016improving} & 84.20 $\pm$ 0.17& 80.79 $\pm$ 0.12 & 111.58 $\pm$ 0.38 & 99.92 $\pm$ 0.30 & 102.41 $\pm$ 0.04 & 96.08 $\pm$ 0.16 \\
SNF~\citep{berg2018sylvester} & 83.32 $\pm$ 0.06 & 80.22 $\pm$ 0.03 & 104.62 $\pm$ 0.29 & 93.82 $\pm$ 0.62 & 99.00 $\pm$ 0.04 & 93.77 $\pm$ 0.03 \\
\hline
VFG (ours) &\textbf{80.80 $\pm$ 0.76} & \textbf{63.66 $\pm$ 0.14} & \textbf{67.26 $\pm$ 0.53} & \textbf{65.74 $\pm$ 0.84}  &\textbf{80.16 $\pm$ 0.73 } & \textbf{78.65 $\pm$ 0.66}\\
\hline
\end{tabular}}\vspace{-0.1in}
\end{table}

\subsection{ELBO and Likelihood}\label{sec:exp:elbo}

We further qualitatively compare our VFG model with existing methods on data distribution learning and  variational inference using three standard datasets. The baselines we compare in this experiment are VAE~\citep{kingma2013auto}, Planer~\citep{rezende2015variational},  IAF~\citep{kingma2016improving}, and SNF~\citep{berg2018sylvester}.
The evaluation datasets and setup are following two standard flow-based variational models, Sylvester Normalizing Flows~\citep{berg2018sylvester} and~\citep{rezende2015variational}.
We use a tree VFG with structure as shown in Figure~\ref{fig:mnist_tree} for three datasets.

\begin{figure}[h]

\centering
       \includegraphics[width=2.5in]{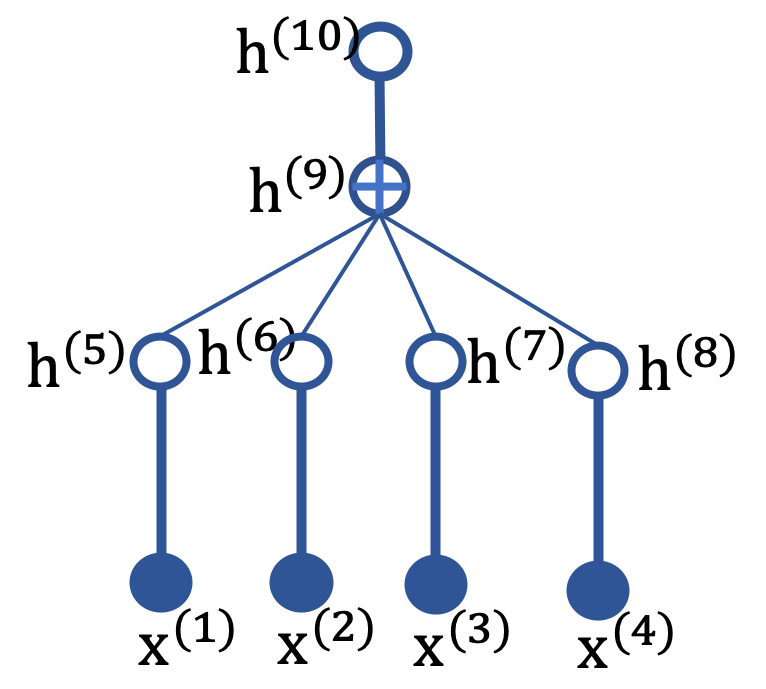}

\vspace{-0.1in}

  \caption{MIST Tree structure.}
    \label{fig:mnist_tree}
\end{figure}

We train the tree VFG  with the following ELBO objective that incorporate a $\beta$ coefficient for the $\mathbf{KL}$ terms.
Empirically, a small $\beta$ yields better ELBO and NLL values, and we set $\beta$ around 0.1 in the experiments.  Recall that
\begin{align} \notag
&\text{ELBO}= \mathcal{L}(\mathbf{x}; \theta)
    = \mathbb{E}_{q(\mathbf{h}^{1:L}|\mathbf{x})}\big[ \log p(\mathbf{x}|\mathbf{h}^{1:L})  \big] - \beta\sum_{l=1}^{L} \mathbf{KL}^l.
\end{align}

Table~\ref{tab:elbo} presents the negative evidence lower bound~(-ELBO) and the estimated negative likelihood~(NLL) for all methods on three datasets: MNIST, Caltech101, and Omniglot.  The baseline methods are VAE based methods  enhanced with  normalizing flows.  They use 16 flows to improve the posterior estimation. SNF is orthogonal Sylvester flow method with a bottleneck of M = 32. We set the VFG coupling block~\citep{Dinh2016DensityEU} number with  $\mathcal{B}=4$, and following~\citep{berg2018sylvester} we run multiple times to get the mean and standard derivation as well. VFG can achieve superior EBLO as well as NLL values on all three datasets compared against  the baselines as given in Table~\ref{tab:elbo}. VFGs can achieve better variational inference and data distribution modeling results (ELBOs and  NLLs) in Table~\ref{tab:elbo} in part due to VFGs' universal approximation power as given in Theorem~\ref{thm:lp_univ}. Also, the intrinsic  approximate invertible property of VFGs ensures the decoder or generative model in a VFG to achieve smaller reconstruction errors for data samples and hence smaller NLL values.

\subsection{Latent Representation Learning on MNIST}\label{sec:exp:mnist}

In this set of experiments, we evaluate VFGs on latent representation learning of the MNIST dataset~\citep{Lecunmnist2010}. We construct a tree  VFG model depicted in Figure~\ref{fig:mnist_tree}.
In the first layer, there are 4 flow functions, and each of them takes $14\times 14$ image blocks as the input.
Thus a $28\times 28$ input image is divided into four $14\times 14$ blocks as the input of VFG model. We use $\mathcal{B}=4$ for all the flows.
The latent dimension for this model is $m=196$.
Following~\citet{Sorrenson2020}, the VFG model is trained with image labels to learn the latent representation of the input data. We set the parameters of $\mathbf{h}^L$'s prior distribution as a function of image label, i.e., $\lambda^L(u)$, where $u$ denotes the image label.
\begin{figure}[h]
    \centering
       \includegraphics[width=3in]{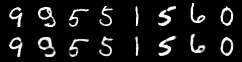}
        \caption{(Top) original MNIST digits. (Bottom) reconstructed images using VFG.}
    \label{fig:reconst}
\end{figure}
In practice, we use $10$ trainable $\lambda^L$s regarding the $10$ digits.
The images in the second row of
Figure~\ref{fig:reconst} are reconstructions of MNIST samples extracted from the testing set, displayed in the first row of the same Figure, using our proposed VFG model.

\vspace{0.1in}

\begin{figure}[b!]
\begin{center}
\mbox{
       \includegraphics[width=2.5in]{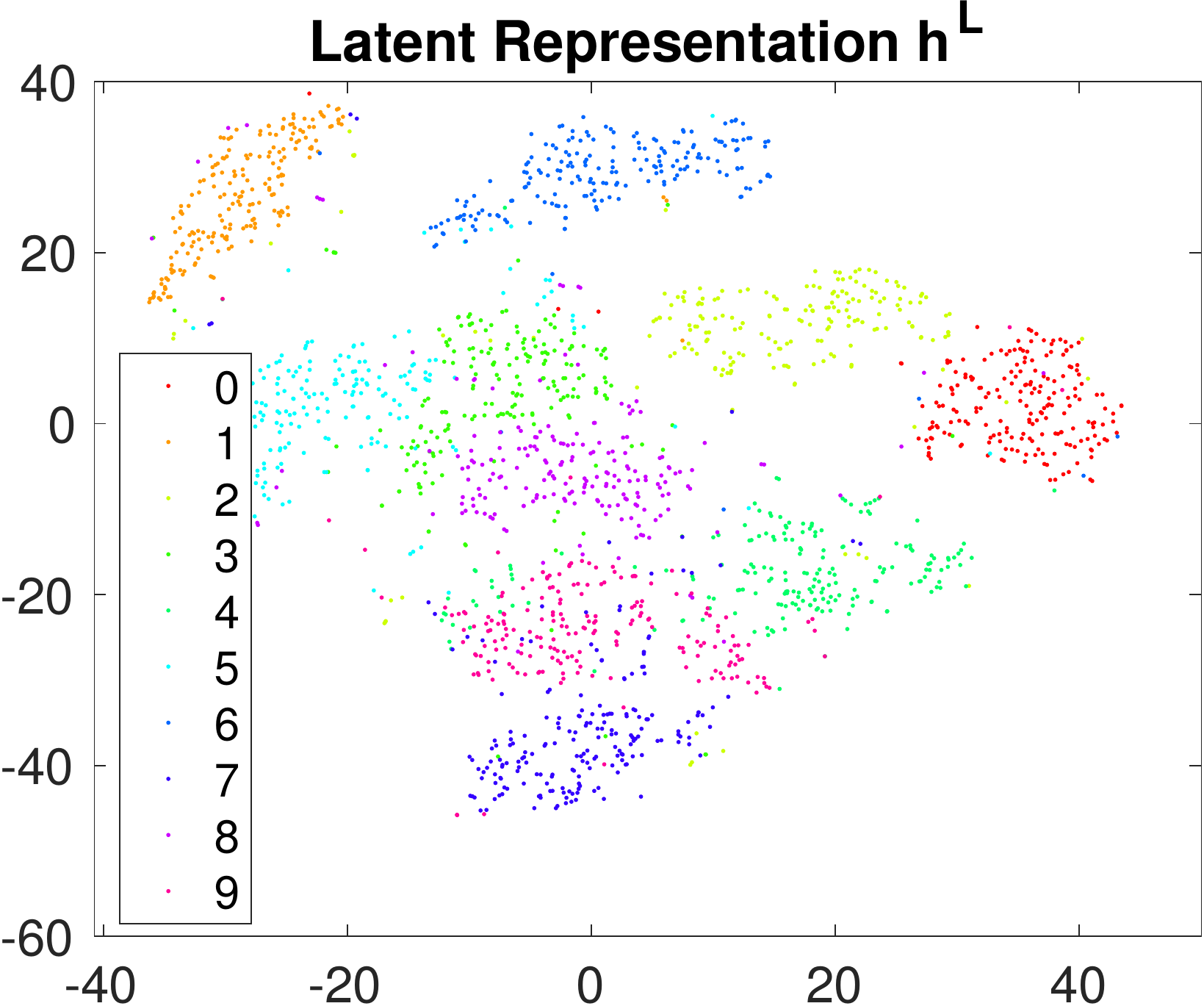}

\hspace{0.1in}

       \includegraphics[width=2.5in]{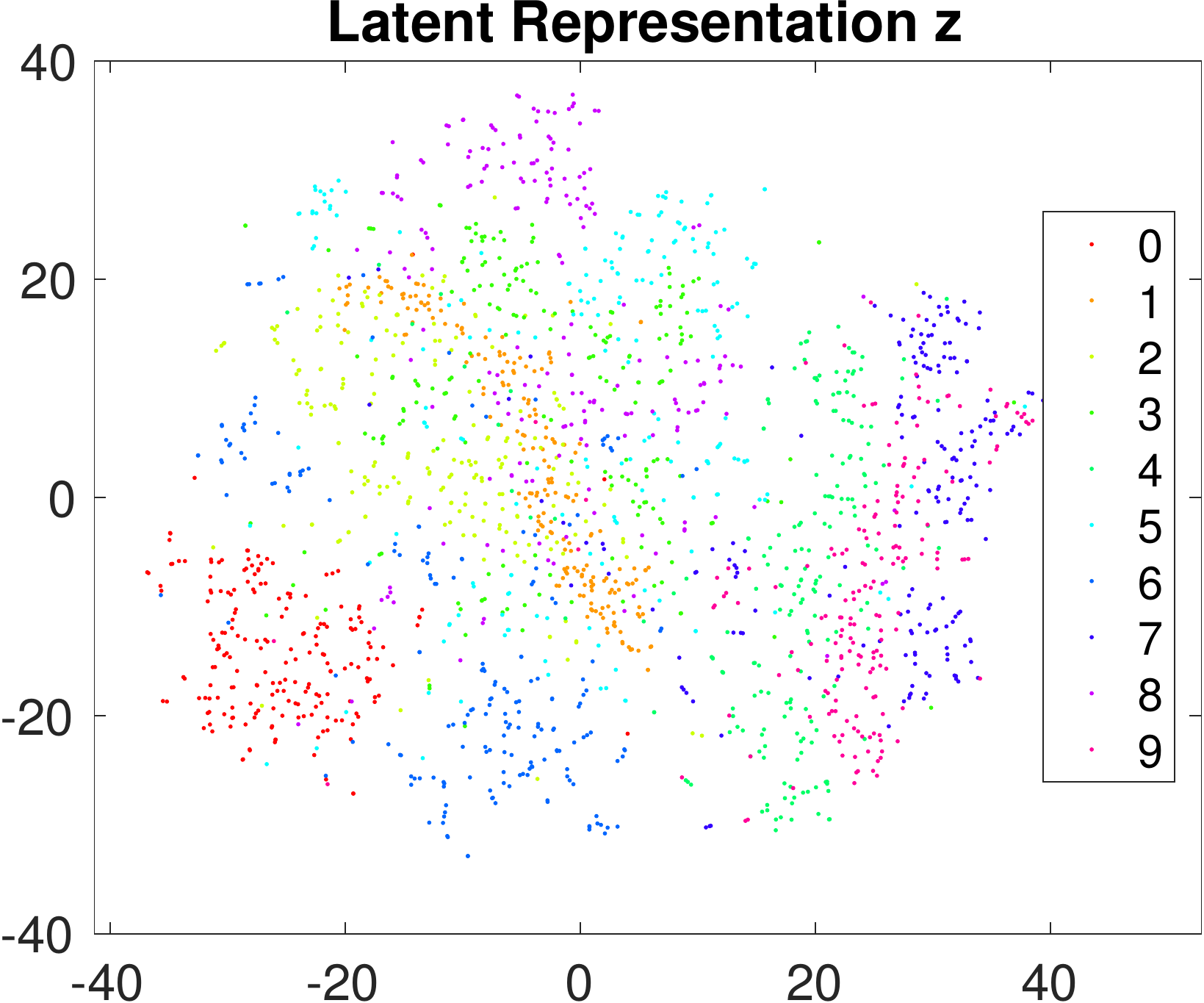}
}
\end{center}

\vspace{-0.1in}

        \caption{  t-SNE of latent variables for VFG (Left) and coupling-layer based~flow~(Right)~on~MNIST.}
    \label{fig:z_tsne}
\end{figure}

Figure~\ref{fig:z_tsne}-Left shows  t-distributed stochastic neighbor embedding (t-SNE)~\citep{maaten2008visualizing} plot of $2,000$ testing images' latent variables learned with our model, and $200$ for each~digit.
Figure~\ref{fig:z_tsne}-Left illustrates that VFG can learn separated latent representations to distinguish different hand-written numbers. For comparison, we also present the results of a baseline model. The baseline model (coupling-based flow) is constructed using the same coupling block and similar number of parameters as VFGs but with $28\times 28$ as the input and latent dimension.  Figure~\ref{fig:z_tsne}-Right gives the baseline  coupling-layer-based flow  training and testing with the same procedures. These show that coupling-based flow cannot  give a clear division between some digits, e.g., 1 and 2, 7 and 9  due to the bias introduced by the high-dimensional redundant latent variables.

\newpage\clearpage

\begin{figure}[b!]

\vspace{-0.1in}

\begin{center}
 \includegraphics[width=4in]{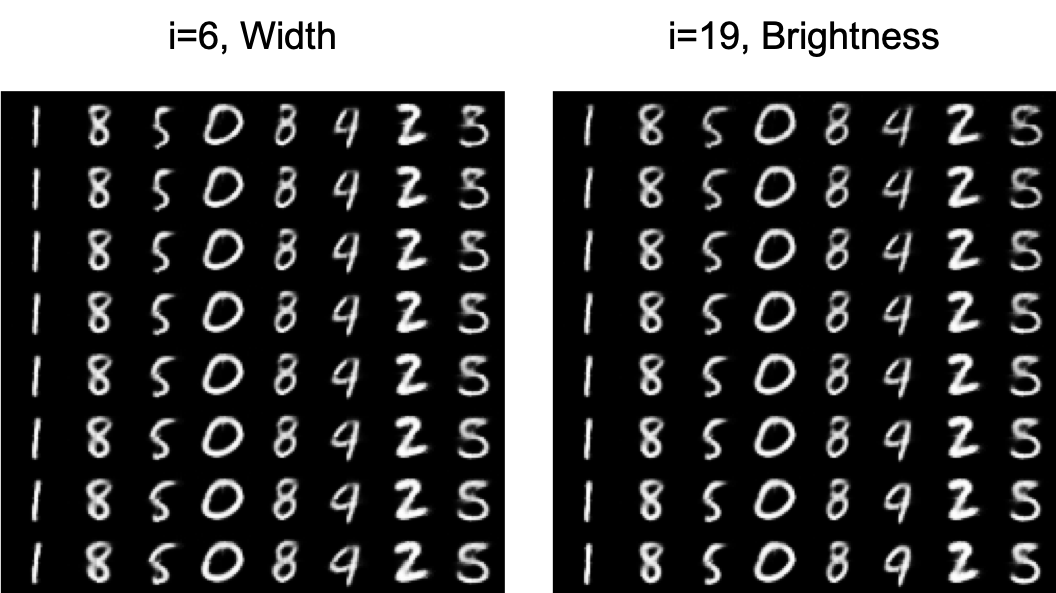}
  \includegraphics[width=4in]{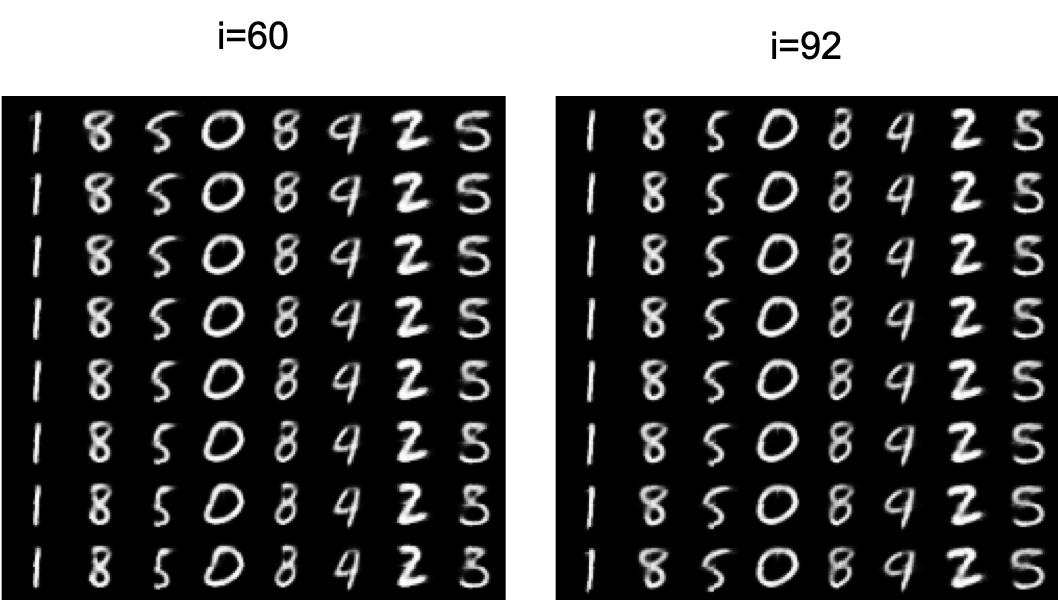}
    \includegraphics[width=4in]{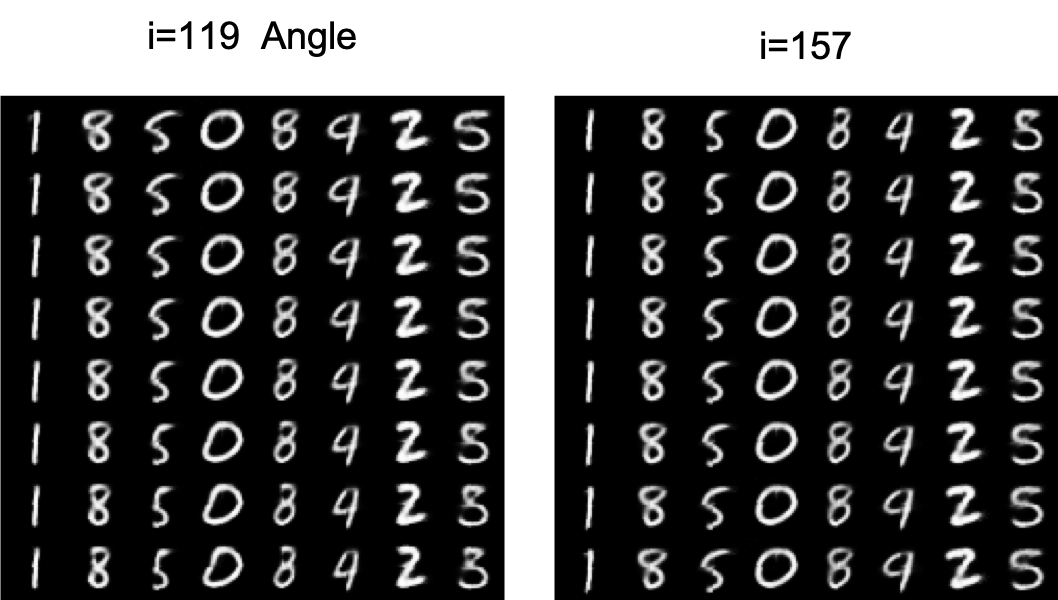}
\end{center}

\vspace{-0.2in}

\caption{MNIST: Increasing each latent variable from a small value to a larger one.}\label{fig:mnist_dis}\vspace{-0.1in}
\end{figure}

To provide a description of the learned latent representation, we first obtain the root latent variables of a set of images through forward message passing. Each latent variable's values are changed increasingly within a range centered at the value of the latent variable obtained from last step.
This perturbation is performed for each image in the set.
Figure~\ref{fig:mnist_dis} shows the change of images by increasing one latent variable from a small value to a larger one. The figure presents some of the latent variables that have obvious effects on images, and most of the $m=196$ variables do not impact the generation significantly. Latent variables $i=6$ and $i=60$ control the digit width. Variable $i=19$ affects the brightness.  $i=92, i=157$ and some of the variables not displayed here control the style of the generated digits.

\section{Discussion}\label{sec:discuss}

One of the motivations for proposing our VFG algorithm is to develop a tractable model that can be used for distribution learning and posterior inference.
As long as the node states in the aggregation nodes are consistent, we can always apply VFGs in order to infer missing values.
We provide more discussion on the structures of VFGs in the sequel.

\subsection{Benefits of Encoder-decoder Parameter Sharing}
There are several advantages for the encoder and decoder to share parameters.
 Firstly, it makes the network's structure simple.
 Secondly, the training and inference can be simplified with concise and simple graph structures.
 Thirdly, by leveraging invertible flow-based functions, VFGs  obtain tighter ELBOs in comparison with VAE based models.
 The intrinsic invertibility introduced by flow functions ensures the decoder or generative model in a VFG  achieves smaller reconstruction errors for data samples and hence smaller NLL values and tighter ELBO. Whereas without the intrinsic constraint of invertibility or any help or regularization from the encoder, VAE-based models have to learn an unassisted mapping function~(decoder) to reconstruct all data samples with the latent variables, and there are always some discrepancy errors in the reconstruction that lead to relatively larger NLL values and hence inferior ELBOs.

\subsection{Structures of VFGs}
 In the experiments, the model structures have been chosen heuristically and for the sake of numerical illustrations. A tree VFG model can be taken as a dimension reduction model that is available for missing value imputation as well. Variants of those structures will lead to different numerical results and at this point, we can not claim any generalization regarding the impact of the VFG structure on the outputs. Meanwhile, learning the structure of VFG is an interesting research problem and is left for future works.  VFG structures could be learned through the regularization of DAG structures~\citep{Zheng2018,wehenkel2021graphical}.

 VFGs rely on minimizing the KL term to achieve \emph{consistency} in  aggregation nodes. As long as the aggregation nodes retain consistency, the model always has a tight ELBO and can be applied to tractable posterior inference.
 According to~\citet{teshima2020coupling}, coupling-based flows are endowed with the universal approximation power.
 Hence, we believe that the consistency of aggregation nodes on a  VFG can be attained with a tight ELBO.

\section{Conclusion}\label{sec:conclusion}

In this paper, we propose VFG, a variational flow graphical model that aims at bridging the gap between  flow-based models and the paradigm of graphical models.
Our VFG model learns  data distribution and latent representation  through message passing between nodes in the model structure.
We leverage the power of invertible flow functions in any general graph structure to simplify the inference step of the latent nodes given some input observations.
We illustrate the effectiveness of our variational model through experiments.
Future work includes applying our VFG model to  relational data structure learning and reasoning.

\newpage\clearpage

\appendix

\noindent \textbf{\huge Appendix}

\section{ELBO of Tree VFGs}\label{appd:tree_elbo}

\begin{figure}[h]
    \centering
    \includegraphics[width=2.0in]{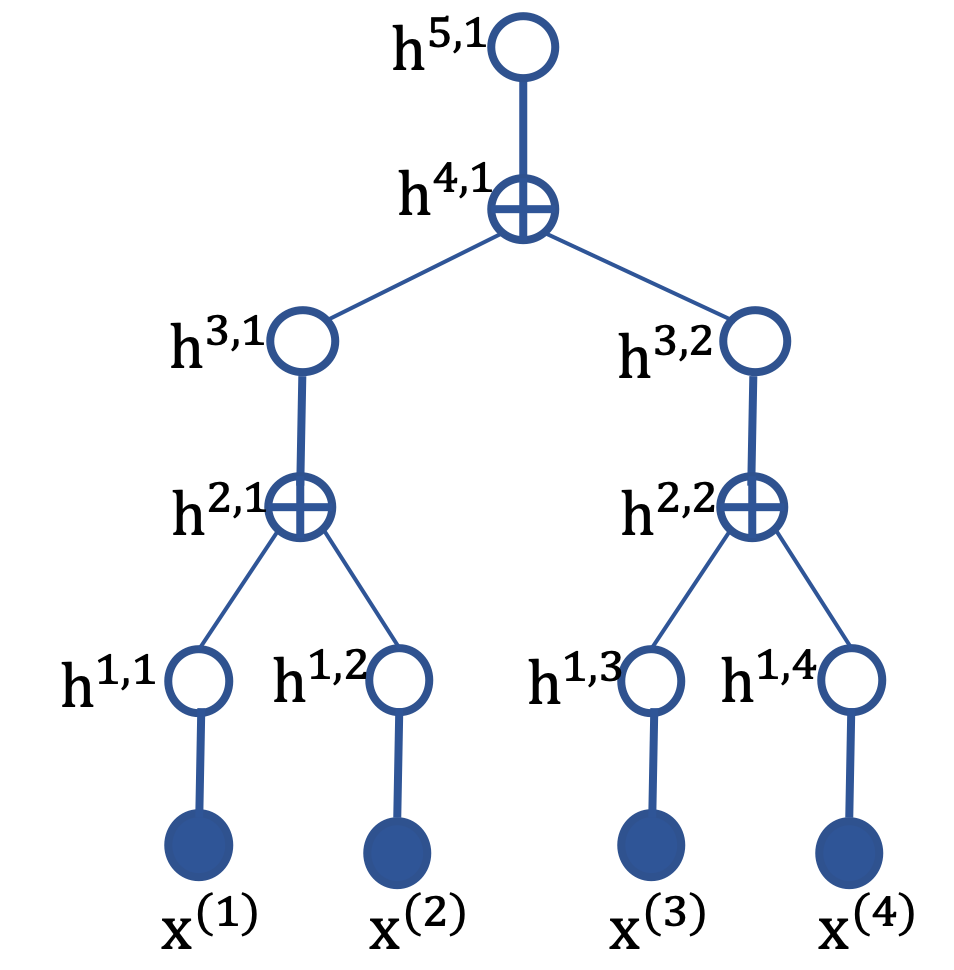}\hspace{0.3in}
    \includegraphics[width=2.5in]{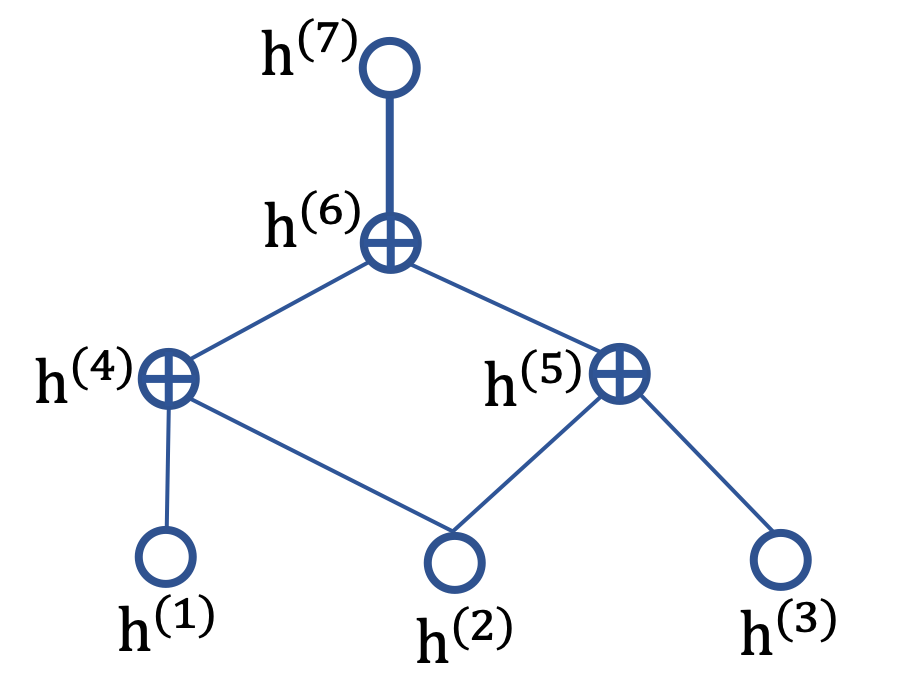}
    \caption{(Left) A  tree VFG with $L=5$ and three aggregation nodes. (Right) A DAG  with inverse topology order \big\{ \{1,2,3\}, \{4,5\}, \{6\},  \{7\} \big\}, and they  correspond to layers 0 to 3. }
    \label{fig:tree_dag_vfg}
\end{figure}

Let each data sample has $k$ sections, i.e., $\mathbf{x} = [\mathbf{x}^{(1)}, ..., \mathbf{x}^{(k)}]$. VFGs are graphical models that can integrate different sections or components of the dataset.  We assume that for each pair of connected nodes, the edge is an invertible flow function.
The vector of parameters for all the edges is denoted by $\theta$.
The forward message passing starts from $\mathbf{x}$ and ends at $\mathbf{h}^L$, and backward message passing in the reverse direction. We start with the hierarchical generative tree network structure illustrated by an example in Figure~\ref{fig:tree_dag_vfg}-Left.
Then the marginal likelihood term of the data reads
\begin{align*}
p(\mathbf{x}| \mathbf{\theta}) = \sum_{\mathbf{h}^1, ..., \mathbf{h}^L} p(\mathbf{h}^L | \theta)p(\mathbf{h}^{L-1} | \mathbf{h}^{L},\theta) \cdot \cdot  \cdot  p(\mathbf{x} | \mathbf{h}^{1}, \theta) \, .
\end{align*}
The hierarchical generative model is given by factorization
\begin{align}\label{eq:prior}
p(\mathbf{h}) =  p( \mathbf{h}^{L})\mathbf{\Pi}_{l=1}^{L-1}p(\mathbf{h}^{l} | \mathbf{h}^{l+1}) .
\end{align}
The probability density function $p(\mathbf{h}^{l-1} | \mathbf{h}^{l})$ in the generative model is modeled with one or multiple invertible normalizing flow functions. The hierarchical posterior~(recognition network) is factorized as
\begin{align}\label{eq:posterior2}
q_{\theta}(\mathbf{h}| \mathbf{x}) =  q(\mathbf{h}^1 | \mathbf{x})  q(\mathbf{h}^2 | \mathbf{h}^1) \cdot \cdot  \cdot  q(\mathbf{h}^{L} | \mathbf{h}^{L-1}).
\end{align}
Draw samples from the generative model~(\ref{eq:prior})
involves sequential conditional sampling from the top of the tree to the bottom, and computation of the recognition model~(\ref{eq:posterior2}) takes the reverse direction. Notice that
\begin{align*} 
q(\mathbf{h}| \mathbf{x}) = q(\mathbf{h}^1 | \mathbf{x})  q(\mathbf{h}^{2:L} | \mathbf{h}^1) \, .
\end{align*}
With the hierarchical structure of a tree, we further have
\begin{align} \label{eq:chain_post}
&q(\mathbf{h}^{l:L}|\mathbf{h}^{l-1}) = q(\mathbf{h}^{l}|\mathbf{h}^{l-1}) q(\mathbf{h}^{l+1:L}|\mathbf{h}^{l}\mathbf{h}^{l-1}) =q(\mathbf{h}^{l}|\mathbf{h}^{l-1}) q(\mathbf{h}^{l+1:L}|\mathbf{h}^{l})  \\ \label{eq:chain_prior}
& p(\mathbf{h}^{l:L})=  p(\mathbf{h}^{l}|\mathbf{h}^{l+1:L})p(\mathbf{h}^{l+1:L})=p(\mathbf{h}^{l}|\mathbf{h}^{l+1})p(\mathbf{h}^{l+1:L})
\end{align}
By leveraging  the conditional independence  in the chain structures of both recognition and generative models, the derivation of trees' ELBO becomes easier.
\begin{align*}
\log p(\mathbf{x}) &= \log \int p(\mathbf{x}|\mathbf{h})p(\mathbf{h}) d \mathbf{h} \\
&= \log \int \frac{q(\mathbf{h}|\mathbf{x})}{q(\mathbf{h}|\mathbf{x})} p(\mathbf{x}|\mathbf{h})p(\mathbf{h}) d \mathbf{h} \\
& \geqslant \mathbb{E}_{q(\mathbf{h}|\mathbf{x})}\big[ \log p(\mathbf{x}|\mathbf{h}) -  \log q(\mathbf{h}|\mathbf{x}) +  \log p(\mathbf{h}) \big]\\ \notag
&= \mathcal{L}(x; \theta).
\end{align*}
The last step is due to the Jensen inequality. With $\mathbf{h} =\mathbf{h}^{1:L} $,
\begin{align} \notag  
&\log p(\mathbf{x})  \geqslant  \mathcal{L}(x; \theta) \\ \notag
=& \mathbb{E}_{q(\mathbf{h}^{1:L}|\mathbf{x})}\big[ \log p(\mathbf{x}|\mathbf{h}^{1:L}) -  \log q(\mathbf{h}^{1:L}|\mathbf{x}) +  \log p(\mathbf{h}^{1:L}) \big] \\ \label{eq:elbo12L}
=&  \underbrace{\mathbb{E}_{q(\mathbf{h}^{1:L}|\mathbf{x})}\big[ \log p(\mathbf{x}|\mathbf{h}^{1:L})\big]}_{\parbox{10.5em}{Reconstruction of data}}  
-  \underbrace{\mathbb{E}_{q(\mathbf{h}^{1:L}|\mathbf{x})}\big[ \log q(\mathbf{h}^{1:L}|\mathbf{x}) - \log p(\mathbf{h}^{1:L}) \big] }_{\textbf{\text{KL}}^{1:L}}
\end{align}
With conditional independence in   the hierarchical structure, we have
$$q(\mathbf{h}^{1:L}|\mathbf{x})=q(\mathbf{h}^{2:L}|\mathbf{h}^1\mathbf{x})q(\mathbf{h}^{1}|\mathbf{x})=q(\mathbf{h}^{2:L}|\mathbf{h}^1)q(\mathbf{h}^{1}|\mathbf{x}).$$
The second term of~(\ref{eq:elbo12L}) can be further expanded as
 \begin{align} \notag
\textbf{\text{KL}}^{1:L} =&  \mathbb{E}_{q(\mathbf{h}^{1:L}|\mathbf{x})}\big[  \log q(\mathbf{h}^{1}|\mathbf{x})   +  \log q(\mathbf{h}^{2:L}|\mathbf{h}^{1})  \\
& - \log p(\mathbf{h}^{1}|\mathbf{h}^{2:L}) - \log p(\mathbf{h}^{2:L})  \big].
\end{align}
Similarly, with conditional independence of the hierarchical latent variables, $ p(\mathbf{h}^{1}|\mathbf{h}^{2:L})= p(\mathbf{h}^{1}|\mathbf{h}^{2})$. Thus
 \begin{align} \notag
\textbf{\text{KL}}^{1:L} =&  \mathbb{E}_{q(\mathbf{h}^{1:L}|\mathbf{x})}\big[  \log q(\mathbf{h}^{1}|\mathbf{x})   - \log p(\mathbf{h}^{1}|\mathbf{h}^{2})  \\  \notag
&+  \log q(\mathbf{h}^{2:L}|\mathbf{h}^{1})- \log p(\mathbf{h}^{2:L})  \big]\\ \notag
=&  \underbrace{\mathbb{E}_{q(\mathbf{h}^{1:L}|\mathbf{x})}\big[  \log q(\mathbf{h}^{1}|\mathbf{x})   - \log p(\mathbf{h}^{1}|\mathbf{h}^{2}) \big]}_{\mathbf{KL}^1}   \\ \notag
& + \underbrace{\mathbb{E}_{q(\mathbf{h}^{1:L}|\mathbf{x})}\big[  \log q(\mathbf{h}^{2:L}|\mathbf{h}^{1})- \log p(\mathbf{h}^{2:L})  \big]}_{\mathbf{KL}^{2:L}}.
\end{align}
We can further expand the $\mathbf{KL}^{2:L}$ term following similar conditional independent rules regarding the tree structure.
At level $l$, we get
$$\textbf{\text{KL}}^{l:L}
= \mathbb{E}_{q(\mathbf{h}^{1:L}|\mathbf{x})}\big[  \log q(\mathbf{h}^{l:L}|\mathbf{h}^{l-1})- \log p(\mathbf{h}^{l:L})  \big].$$
With~(\ref{eq:chain_post}) and~(\ref{eq:chain_prior}), it is easy to show that
 \begin{align} \label{eq:kl_lL}
\textbf{\text{KL}}^{l:L}
=&  \underbrace{\mathbb{E}_{q(\mathbf{h}^{1:L}|\mathbf{x})}\big[  \log q(\mathbf{h}^{l}|\mathbf{h}^{l-1})   - \log p(\mathbf{h}^{l}|\mathbf{h}^{l+1}) \big]}_{\mathbf{KL}^l} \notag \\
&+ \underbrace{\mathbb{E}_{q(\mathbf{h}^{l:L}|\mathbf{x})}\big[  \log q(\mathbf{h}^{l+1:L}|\mathbf{h}^{l})- \log p(\mathbf{h}^{l+1:L})  \big]}_{\mathbf{KL}^{l+1:L}}.
\end{align}
The ELBO~(\ref{eq:elbo12L}) can be written as
\begin{align} \label{eq:elbo0}
\mathcal{L}(\mathbf{x}; \theta) = \mathbb{E}_{q(\mathbf{h}^{1:L}|\mathbf{x})}\big[ \log p(\mathbf{x}|\mathbf{h}^{1:L})  \big] - \sum_{l=1}^{L-1} \mathbf{KL}^l -\mathbf{KL}^L.
\end{align}
When $1\leqslant l \leqslant L-1$
 \begin{align} \label{eq:kl_l}
 \mathbf{KL}^l=\mathbb{E}_{q(\mathbf{h}^{1:L}|\mathbf{x})}\big[  \log q(\mathbf{h}^{l}|\mathbf{h}^{l-1})   - \log p(\mathbf{h}^{l}|\mathbf{h}^{l+1}) \big].
 \end{align}
According to conditional independence, the expectation regarding variational distribution layer $l$ just depends on layer $l-1$. We can simplify the expectation each term of~(\ref{eq:elbo0}) with the default assumption that all latent variables are generated regarding data sample $\mathbf{x}$.  Therefore the ELBO~(\ref{eq:elbo0}) can be simplified as
 \begin{align} \label{eq:elbo1}
\mathcal{L}(\mathbf{x}; \theta) = \mathbb{E}_{q(\mathbf{h}^{1}|\mathbf{x})}\big[ \log p(\mathbf{x}|\widehat{\mathbf{h}}^{1})  \big] - \sum_{l=1}^{L} \mathbf{KL}^l.
\end{align}
The $\mathbf{KL}$
term~(\ref{eq:kl_l}) becomes
\begin{align*}
 \mathbf{KL}^l=\mathbb{E}_{q(\mathbf{h}^{l}|\mathbf{h}^{l-1})}\big[  \log q(\mathbf{h}^{l}|\mathbf{h}^{l-1})   - \log p(\mathbf{h}^{l}|\widehat{\mathbf{h}}^{l+1}) \big].
 \end{align*}
When $l=L$,
$$\mathbf{KL}^L =  \mathbb{E}_{q(\mathbf{h}^{L}|\mathbf{h}^{L-1})}\big[  \log q(\mathbf{h}^{L}|\mathbf{h}^{L-1})- \log p(\mathbf{h}^{L})  \big].$$

\section{ELBO of DAG VFGs}\label{appd:dag_elbo}
Note that if we reverse the edge directions in a DAG, the resulting graph is still a DAG graph.
The nodes can be listed in a topological order regarding the DAG structure as shown in Figure~\ref{fig:tree_dag_vfg}-Right.

By taking the topology order as the layers in tree structures, we can derive the ELBO for DAG structures.
Assume the DAG structure has $L$ layers, and the root nodes are in layer $L$.
We denote by $\mathbf{h}$ the vector of latent variables, then following~(\ref{eq:elbo12L}) we develop the ELBO as
\begin{align}  \label{eq:dag_elbo}
\log p(\mathbf{x}) & \geqslant  \mathcal{L}(x;\theta) \\ \notag
&=   \mathbb{E}_{q(\mathbf{h} | \mathbf{x})} \bigg[ \log  \frac{p(\mathbf{x}, \mathbf{h})}{q(\mathbf{h}|\mathbf{x})}  \bigg]  \\ \notag
&= \underbrace{ \mathbb{E}_{q(\mathbf{h} | \mathbf{x})} \bigg[ \log  p(\mathbf{x} |\mathbf{h})  \bigg] }_{  \parbox{10.5em}{Reconstruction of the data}}  -  \underbrace{  \mathbb{E}_{q(\mathbf{h}| \mathbf{x})} \bigg[  \log q(\mathbf{h}|\mathbf{x}) - \log p( \mathbf{h}) \bigg] }_{\textbf{\text{KL}}} \, .   \notag
\end{align} 
Similarly the KL term can be expanded as in the tree structures.
For nodes in layer $l$
\begin{align*}
\textbf{\text{KL}}^{l:L}
= &\mathbb{E}_{q(\mathbf{h}^{1:L}|\mathbf{x})}\big[  \log q(\mathbf{h}^{l:L}|\mathbf{h}^{1:l-1})- \log p(\mathbf{h}^{l:L})  \big].
\end{align*} 
Note that $ch(l)$ may include nodes from layers lower than $l-1$, and $pa(l)$ may include nodes from layers higher than $l$.
Some nodes in $l$ may not have parent. Based on conditional independence with the topology order of a DAG, we have
\begin{align} \label{eq:dag_chain_q}
&q(\mathbf{h}^{l:L}|\mathbf{h}^{1:l-1})\\ \notag
=&q(\mathbf{h}^{l}|\mathbf{h}^{1:l-1})q(\mathbf{h}^{l+1:L}|\mathbf{h}^{l})\\ \label{eq:dag_chain_p}
=&q(\mathbf{h}^{l}|\mathbf{h}^{1:l-1})q(\mathbf{h}^{l+1:L}|\mathbf{h}^{1:l}) p(\mathbf{h}^{l:L}) \\ \notag
=& p(\mathbf{h}^{l}|\mathbf{h}^{l+1:L})p(\mathbf{h}^{l+1:L})
\end{align}
Following~(\ref{eq:kl_lL}) and with~(\ref{eq:dag_chain_q}-\ref{eq:dag_chain_p}), we have
 \begin{align} \notag
\textbf{\text{KL}}^{l:L}
=&  \mathbb{E}_{q(\mathbf{h}^{1:L}|\mathbf{x})}\big[  \log q(\mathbf{h}^{l}|\mathbf{h}^{1:l-1})   - \log p(\mathbf{h}^{l}|\mathbf{h}^{l+1:L}) \big] \\ \notag
&+ \underbrace{\mathbb{E}_{q(\mathbf{h}^{l:L}|\mathbf{x})}\big[  \log q(\mathbf{h}^{l+1:L}|\mathbf{h}^{1:l})- \log p(\mathbf{h}^{l+1:L})  \big]}_{\mathbf{KL}^{l+1:L}}.
\end{align}
Furthermore,
\begin{align*}
q(\mathbf{h}^{l}|\mathbf{h}^{1:l-1})=q(\mathbf{h}^{l}|\mathbf{h}^{ch(l)}), \quad  \quad   p(\mathbf{h}^{l}|\mathbf{h}^{l+1:L}) = p(\mathbf{h}^{l}|\mathbf{h}^{pa(l)}).
\end{align*}
Hence,
 \begin{align} \label{eq:dag_kl_lL}
\textbf{\text{KL}}^{l:L}
=&  \underbrace{\mathbb{E}_{q(\mathbf{h}^{1:L}|\mathbf{x})}\big[ \log q(\mathbf{h}^{l}|\mathbf{h}^{ch(l)})  - \log p(\mathbf{h}^{l}|\mathbf{h}^{pa(l)}) \big]}_{\textbf{\text{KL}}^{l}} +\textbf{\text{KL}}^{l+1:L}
\end{align} 
For nodes in layer $l$,
\begin{align} \notag
\textbf{\text{KL}}^{l} =& \sum_{i \in l} \underbrace{\mathbb{E}_{q(\mathbf{h}^{1:L}|\mathbf{x})}\big[  \log q(\mathbf{h}^{(i)}|\mathbf{h}^{ch(i)})  - \log p(\mathbf{h}^{(i)}|\mathbf{h}^{pa(i)}) \big]}_{\textbf{\text{KL}}^{(i)}} .
\end{align}
Recurrently applying~(\ref{eq:dag_kl_lL}) to (\ref{eq:dag_elbo}) yields
\begin{align}\notag 
\mathcal{L}(\mathbf{x}; \theta) =& \mathbb{E}_{q(\mathbf{h}|\mathbf{x})}\big[ \log p(\mathbf{x}|\mathbf{h})  \big] -  \sum_{i \in \mathcal{V}  \setminus  \mathcal{R}_{ \mathbb{G} }} \textbf{\text{KL}}^{(i)} \\ \notag
&-    \sum_{i \in  \mathcal{R}_{ \mathbb{G} }  }  \textbf{\text{KL}}\big(q(\mathbf{h}^{(i)} | \mathbf{h}^{ch(i)} )   || p(\mathbf{h}^{(i)})  \big) .
\end{align}
For node $i$,
$$\textbf{\text{KL}}^{(i)} = \mathbb{E}_{q(\mathbf{h}|\mathbf{x})}\big[  \log q(\mathbf{h}^{(i)}|\mathbf{h}^{ch(i)})  - \log p(\mathbf{h}^{(i)}|\mathbf{h}^{pa(i)}) \big].$$

\newpage

\bibliographystyle{plainnat}
\bibliography{ref}

\end{document}